\documentclass[twoside]{article}
\usepackage{amsmath,amsfonts,bm}

\def\eqref#1{equation~\ref{#1}}
\def\1{\bm{1}}

\DeclareMathAlphabet{\mathsfit}{\encodingdefault}{\sfdefault}{m}{sl}
\SetMathAlphabet{\mathsfit}{bold}{\encodingdefault}{\sfdefault}{bx}{n}

\usepackage{hyperref}
\usepackage{url}
\usepackage{algorithmic}

\usepackage[utf8]{inputenc} % allow utf-8 input
\usepackage[T1]{fontenc}    % use 8-bit T1 fonts
\usepackage{url}            % simple URL typesetting
\usepackage{booktabs}       % professional-quality tables
\usepackage{amsfonts}       % blackboard math symbols
\usepackage{nicefrac}       % compact symbols for 1/2, etc.
\usepackage{microtype}      % microtypography
\usepackage{xcolor}         % colors
\usepackage{amsmath}
\usepackage{amssymb}
\usepackage{mathtools}
\usepackage{amsthm}
\usepackage[table]{xcolor}
\newcommand{\bbw}{{\bf W}}
\theoremstyle{plain}
\newtheorem{theorem}{Theorem}[section]
\newtheorem{proposition}[theorem]{Proposition}
\newtheorem{lemma}[theorem]{Lemma}

\theoremstyle{definition}
\newtheorem{definition}[theorem]{Definition}

\theoremstyle{remark}
\newtheorem{remark}[theorem]{Remark}
\newtheorem*{proposition*}{Proposition}
\usepackage{bm}
\DeclareMathOperator{\erf}{erf}
\usepackage[textsize=tiny]{todonotes}
\usepackage{newfloat}
\usepackage{listings}
\usepackage[preprint]{aistats2026}
\usepackage[numbers]{natbib}

\newcommand{\bW}{\mathbf{W}}
\newcommand{\bz}{\mathbf{z}}
\newcommand{\bomega}{\boldsymbol{\omega}}
\newcommand{\btheta}{\boldsymbol{\theta}}
\newcommand{\bA}{\boldsymbol{A}}
\newcommand{\bB}{\boldsymbol{B}}

\begin{document}

\twocolumn[

\aistatstitle{Train Less, Infer Faster: Efficient Model Finetuning and Compression via Structured Sparsity}

\aistatsauthor{ Jonathan Svirsky \And Yehonathan Refael \And  Ofir Lindenbaum }

\aistatsaddress{ Bar Ilan University \\ jonathan.svirsky@biu.ac.il \And  Tel Aviv University \\  refaelkalim@mail.tau.ac.il \And Bar Ilan University \\ ofir.lindenbaum@biu.ac.il } ]

\begin{abstract}
Fully finetuning foundation language models (LMs) with billions of parameters is often impractical due to high computational costs, memory requirements, and the risk of overfitting. Although methods like low-rank adapters help address these challenges by adding small trainable modules to the frozen LM, they also increase memory usage and do not reduce inference latency. We uncover an intriguing phenomenon: sparsifying specific model rows and columns enables efficient task adaptation without requiring weight tuning. We propose a scheme for effective finetuning via sparsification using training stochastic gates, which requires minimal trainable parameters, reduces inference time, and removes 20--40\% of model parameters without significant accuracy loss. Empirical results show it outperforms recent finetuning baselines in efficiency and performance. Additionally, we provide theoretical guarantees for the convergence of this stochastic gating process, and show that our method admits a simpler and better-conditioned optimization landscape compared to LoRA. Our results highlight sparsity as a compelling mechanism for task-specific adaptation in LMs.
\end{abstract}

\section{Introduction}
Foundation language models (LMs) have transformed natural language processing by enabling a wide variety of powerful and versatile applications. Pre-trained on extensive amounts of text data, these models demonstrate strong performance in tasks such as text generation, translation, and sentiment analysis. Nonetheless, fine-tuning is often essential to customize these models for specific applications. Finetuning enables the model to adapt to the nuances of a particular task by updating its parameters based on a smaller, task-specific dataset. However, finetuning presents challenges, particularly when task-specific data is limited, which can constrain the model's ability to effectively adapt to the target task and may lead to overfitting or suboptimal performance. Despite these challenges, finetuning remains crucial in leveraging the full potential of large language models for specialized applications.

\begin{figure}[t]
\centering
\includegraphics[width=.8\linewidth]{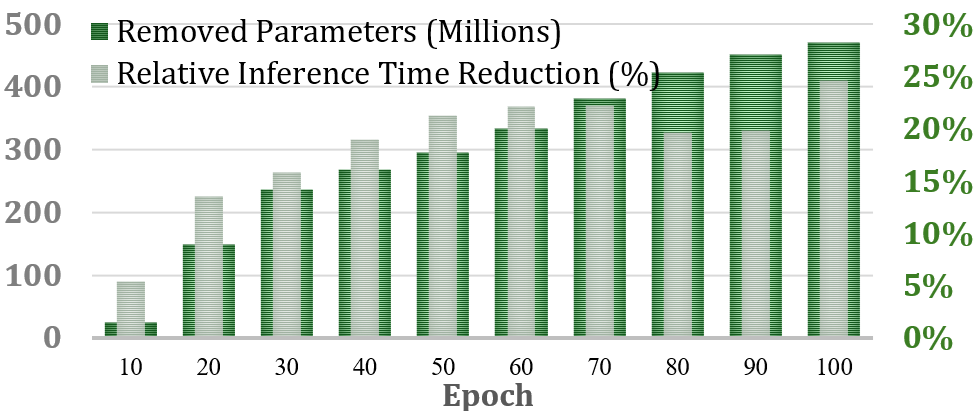}
% \vskip -0.1in
\caption{CPU inference time reduction (\%) and number of removed parameters on the MRPC validation set while finetuning our method on the Llama3.2-1B backbone. See Section~\ref{sec:inference_time} for details.}
\label{fig:speedup}
\vskip -0.2in
\end{figure}

Recently, several methods have been proposed to efficiently finetune foundation language models, addressing the challenges of limited data and the high computational cost of updating the full set of model parameters. One such approach is LoRA (Low-Rank Adaptation) \citep{hu2021lora}, which introduces a more efficient method for finetuning models by freezing the original parameters and replacing them with trainable, low-rank matrices in each transformer layer. This technique significantly reduces the number of parameters that need to be updated during finetuning, making the process both faster and less resource-intensive. Most recent efforts in making finetuning more efficient \citep{zhang2023lora, chavan2023one, xu2023qa, li2022parameter, lin2024nora, balazy2024lora, hu2021lora,refael2025sumo,refael2025lorenza} focus on adding new parameters while keeping the base model frozen, ensuring that the original pre-trained knowledge is retained \citep{rozner2024knowledge}.

When training models using low-rank adapters \citep{kopiczko2023vera, zhang2023lora, lin2024nora, balazy2024lora, hu2021lora}, the number of optimized parameters is reduced, leading to faster convergence during finetuning. However, the inference runtime and memory demands remain the same, as they still depend on the dimensions of the base model parameters, which are frozen during the adaptation. Several methods have been proposed for LM compression, such as quantization and pruning, to reduce the inference time for adapted models. Quantization reduces the memory usage of language models by converting their parameters into lower-bit data types \citep{bondarenko2024low, lin2024awq}. Although quantization reduces the memory consumption of language models, its speedup advantages rely on specialized framework support, which limits its flexibility and adaptability. 

Pruning \citep{ma2023llm} aims to improve inference efficiency in language models by removing unimportant parameters. Structured pruning \citep{xia2022structured,refael2024learningklevelstructuredsparse} eliminates coherent blocks of parameters or entire model dimensions, leading to broad improvements in inference efficiency by reducing computational overhead. However, these methods often require substantial additional training time and involve increasing the number of trainable parameters to facilitate effective self-distillation of knowledge from the layers of a large language teacher model to the smaller student model \citep{hinton2015distilling, xia2022structured}. 

In this work, we propose a simple yet effective approach for adapting language models through structured sparsification, which achieves inference-time efficiency by pruning the overparameterized model and updating only lightweight bias parameters. Instead of learning new parameters or updating main model weights, we learn a set of compact binary masking vectors (gates) that selectively activate a sparse subset of the model's existing weights. This approach leverages the inherent redundancy in overparameterized models, demonstrating that effective task adaptation can be achieved by identifying and utilizing the relevant subnetwork within the frozen base model. The outcome is a highly efficient adaptation mechanism that restructures the original model through sparsification, enabling impressive task-specific performance and inference-time speedup, as presented in Figure \ref{fig:speedup}. Moreover, the proposed method is broadly applicable to both downstream tasks and pretraining scenarios. 

Our gating mechanism is trained end-to-end, embedding task-specific information directly during the optimization process. This eliminates the need for post-training pruning or prolonged finetuning, resulting in models that are both compact and effective, particularly well-suited for deployment in resource-constrained environments.

We evaluate our method on Transformer-based models and demonstrate that up to 40\% of the base model’s parameters can be deactivated with minimal performance loss compared to full finetuning and several baselines. In the following sections, we describe the method in detail, present empirical results, and provide a theoretical analysis of convergence, including a comparison of the optimization landscape of our gating approach versus LoRA, showing that our method enjoys a simpler and better-conditioned geometry.

\section{Related Work}

\subsection{Low-Rank Adaptation}

Low-rank adaptation techniques aim to efficiently finetune foundation language models (LMs) under resource constraints by updating a small subset of parameters. This is typically achieved by introducing trainable components such as adapter layers \citep{pfeiffer2020adapterfusion, houlsby2019parameter}, prompt or prefix embeddings \citep{lester2021power, li2021prefix}, or by imposing low-rank structures on the gradients during training \citep{refael2024adarankgrad}. A widely adopted method, LoRA \citep{hu2021lora}, only trains low-rank matrices to reduce memory requirements. However, LoRA still introduces additional parameters and maintains the original model dimensionality, offering no compression benefits during inference. Moreover, it requires selecting a fixed-rank hyperparameter, a choice based on heuristics that often results in suboptimal performance. Other approaches, such as SparseAdapter \citep{he2022sparseadapter}, dynamically adjust adapter sparsity during training, while AdaLoRA \citep{zhang2023adalora} gradually prunes singular values to reduce the number of updated parameters. Despite their efficiency during training, these methods do not yield inference-time benefits, as the underlying base model remains uncompressed and fully active.

\subsection{Memory efficient optimizers.}
 Another direction for memory reduction does not restrict the set of trainable parameters but instead optimizes the training methods, with notable examples including AdaRankGrad, GaLore, Fira, Flora, Adam-mini, GaLore-mini, LDAdam, GoLore, LoQT, Apollo, and SUMO \citep{ refael2025adarankgrad, zhao2024galore, zhu2024apollo, chen2024fira, Hao2024FloraLA, zhang2024adam, robert2025ldadama, loeschcke2024,refael2025sumo}, integrating low-rank gradient projections in optimization.

\subsection{Finetuning with Pruning}
Pruning refers to the process of reducing a model’s size by removing weights or structures deemed unnecessary for its performance. There are two main frameworks for pruning models during finetuning: \textit{structured} and \textit{unstructured}. Unstructured pruning \citep{sanh2020movement, fang2024maskllm} removes the least important parameters in a model without following any specific order. In contrast, structured pruning \citep{xia2022structured, zhao2024apt} eliminates entire blocks, rows, or columns from the weight matrices. Additionally, a post-training pruning method proposed by \citep{frantar2023sparsegpt} aims to prune finetuned models with minimal additional costs, but it requires initialization from fully finetuned models. In contrast, our approach jointly performs task adaptation and structured pruning during training, allowing the model to specialize while removing up to 40\% of the base parameters—without the need for full finetuning or post-hoc pruning. This results in a model that is both task-specific and resource-efficient, reducing inference cost with no additional training overhead.

\subsection{Finetuning with Adaptive Pruning}

A recent work that closely aligns with our goal was proposed by \citet{zhao2024apt}, where the base model parameters are pruned concurrently with adapter training. Their approach introduces additional low-rank matrices, similar in spirit to LoRA \citep{he2022sparseadapter}, which must also be trained. However, this design comes with substantial overhead: 5 times longer than standard full-model finetuning, and compression is achieved through a costly sorting-and-binary-search procedure over weight blocks. Moreover, because the rank of the low-rank adaptation weights is adaptive, the optimization memory footprint remains heavy, reaching about 70\% of that required for full-model finetuning. By contrast, our method introduces only lightweight gate vectors, which require orders of magnitude fewer parameters to optimize, incur negligible training-time overhead, and achieve compression without resorting to expensive search procedures. This makes our approach both more memory-efficient and more computationally practical than prior methods. In addition, the gates are trained end-to-end and integrate seamlessly into the training pipeline, making the approach simple to implement and deploy.

\section{Problem Formulation}
Assume we are given a pre-trained large language model $P_{\mathbf{\Theta}} (\mathbf{y}|\mathbf{x})$ parametrized by $\mathbf{\Theta}$ based on the Transformer architecture \citep{vaswani2017attention}. The goal of finetuning is to adapt this pre-trained model for downstream natural language understanding tasks, such as question answering or sentiment analysis. The downstream task is represented by a training dataset of context-target pairs: $\mathcal{Z} = \{(\mathbf{x}_i, \mathbf{y}_i)\}_{i\in[N]}$, where $\mathbf{x}_i$ is a sequence and $\mathbf{y}_i$ is a target label. For example, in the question-answering task (QQP) in the GLUE benchmark \citep{wang2019superglue}, $\mathbf{x}_i$ is a question, and $\mathbf{y}_i$ is its corresponding answer. 
\begin{figure*}[t]
\centering
\includegraphics[width=.75\linewidth]{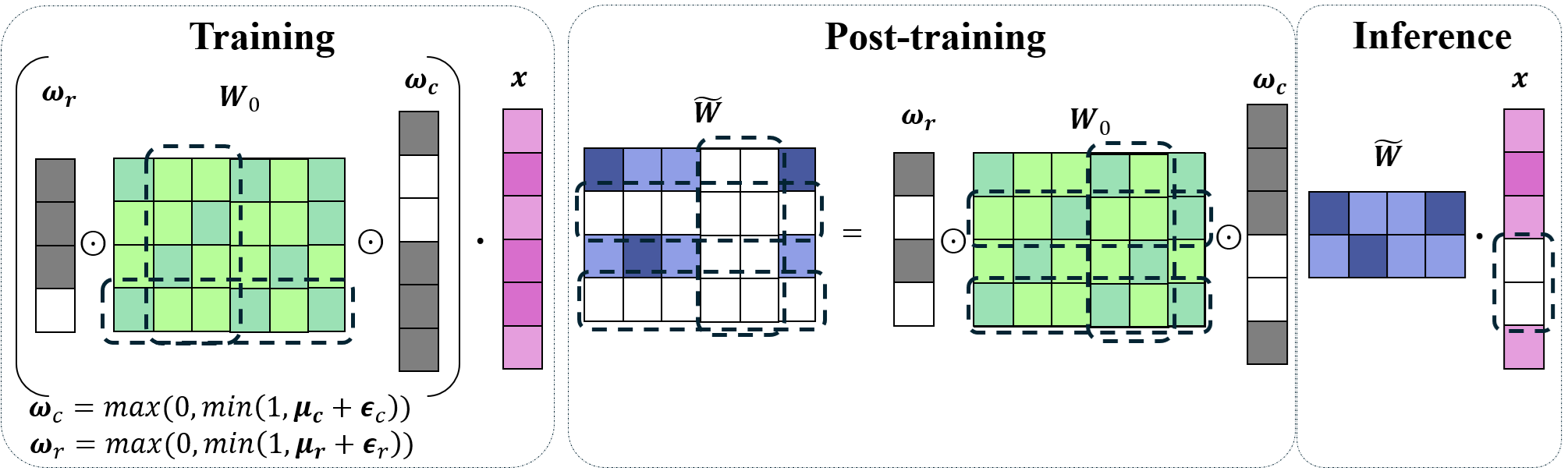}
\vskip -0.1in
\caption{ Overview of FineGates: Our method introduces structured sparsity in LM finetuning by training lightweight row and column gating vectors ($\bm{\omega}_c, \bm{\omega}_r$). These gates selectively retain the most informative weight dimensions, enabling efficient adaptation without modifying the base model’s parameters. Unlike LoRA and other PEFT methods, which introduce additional trainable matrices, FineGates directly optimizes sparsification and updates biases, thereby reducing memory overhead and inference time while maintaining task performance.
}
\label{fig:illusration}
\vskip -0.15in
\end{figure*}
To finetune the whole model parameters (full finetuning), the model is initialized to pre-trained weights $\mathbf{\Theta}_0$ and updated to $\mathbf{\Theta}_0 + \Delta \mathbf{\Theta}$ by repeatedly following the gradient updates to maximize the conditional language modeling objective:
\begin{equation}
    \max_{\mathbf{\Theta}} \sum_{(\mathbf{x},\mathbf{y}) \in \mathcal{Z}} \sum_{t=1}^{|\mathbf{y}|} \log (P_{\mathbf{\Theta}} (y_t|\mathbf{x},\mathbf{y}_{<t})).
\end{equation}

In low-rank adaptation methods the task-specific parameter update $\Delta \mathbf{\Theta} = \Delta \mathbf{\Theta} (\mathbf{\Gamma})$ is further encoded by a much smaller-sized set of parameters $\mathbf{\Gamma}$ with $|\mathbf{\Gamma}|  \ll |\mathbf{\Theta}_0|$. The task of finding $\Delta \mathbf{\Theta}$ thus becomes optimizing over $\mathbf{\Gamma}$ and not $\mathbf{\Theta}$,
\begin{equation}
    \max_{\mathbf{\Gamma}} \sum_{(\mathbf{x},\mathbf{y}) \in \mathcal{Z}} \sum_{t=1}^{|\mathbf{y}|} \log (P_{\mathbf{\Theta}_0 + \Delta \mathbf{\Theta}(\mathbf{\Gamma})} (y_t|\mathbf{x},\mathbf{y}_{<t})).
\end{equation}
While beneficial for preserving a common base model across tasks, this approach retains unnecessary information in the base model and still requires a forward pass through the large number of parameters during inference.

\paragraph{Finetuning by sparsification.} In this work, we propose incorporating \textit{gates} vectors \(\bm{\omega}_{r}, \bm{\omega}_{c}\) such that $\bm{\omega}_{r}$ multiplies the rows of the weight matrix and $\bm{\omega}_{c}$ multiplies its columns in an element-wise way, thereby sparsifying rows and columns. Importantly, the entries of these gate vectors are binary (\(0/1\)), directly enforcing structured sparsity. The parameters that construct these vectors for all adapted layers in the base model are denoted by  \(\bm{\Omega}\). This approach modifies the base model parameters by replacing \(\Delta \mathbf{\Theta}\) with our gate vectors. Our method introduces the concept of \textit{structured sparsity} in the base model \citep{wen2016learning}, aiming to exclude entire columns or rows from the weight matrices. Therefore, assuming that the number of compressed weight matrices in the base model is \(L\), the goal of the finetuning task becomes:

% {\fontsize{9.5pt}{9.5pt}\selectfont
\begin{align}\label{def::total_loss}
    \mathcal{L} = & \min_{\bm{\omega}}  \Bigg[- \sum_{(\mathbf{x},\mathbf{y}) \in \mathcal{Z}} \sum_{t=1}^{|\mathbf{y}|} \log \left(P_{\bm{\omega}_r \odot\mathbf{\Theta}_0 \odot \bm{\omega}_c} (y_t|\mathbf{x},\mathbf{y}_{<t})\right) + \nonumber \\
    &  + \lambda \frac{1}{L}\sum_{i=1}^L \left[ \max (||\bm{\omega}_r^i||_0, s_r^i ) + \max (||\bm{\omega}_c^i||_0,  s_c^i ) \right] \Bigg] ,
\end{align}
% }
where the parameter \(\lambda\) represents the magnitude of structured sparsity regularization. The term \(||\cdot||_0\) refers to the \(\ell_0\) norm, while \(s_r^i\) and \(s_c^i\) denote the target sparsity ratios we aim to achieve. These ratios are defined as the number of zero parameters divided by the total number of parameters in the weight matrix \(i\). Additionally, \(\odot\) signifies the element-wise product between the gates at specific indices and the corresponding column or row vectors at those same indices.

To clarify, structured sparsity is achieved by training two gate vectors, $\bm{\omega}_r,\bm{\omega}_c$, where each element scales an entire row or column of a given weight matrix. These gate vectors are optimized without any additional low-rank weight matrices. Once trained, the base model weights of the adapted layers are compressed by simply multiplying the gate vectors with the corresponding rows and columns of the weight matrix. This sparsification mechanism effectively reduces both memory and time complexity in the attention layers.

Our empirical results indicate that it is possible to reduce the model size while increasing its efficiency, and still provide accurate predictions, even when training only a small set of gates.

\section{The Method}

Consider the Transformer architecture \citep{vaswani2017attention} composed of $L$ blocks, and each block consists of a multi-head self-attention (MHA) layer and a feed-forward (FFN) layer. An MHA layer with $N_h$ heads takes an input $\mathbf{X} \in \mathbb{R}^{b \times t \times d}$, where $b$ is the batch size, $t$ is the sequence length (number of tokens), and $d$ is the embedding dimension. This input is either the output of the former encoder layer or the embedding layer. The MHA layer outputs:
\begin{equation*}
\text{MHA}(\mathbf{X}) = \sum_{i=1}^{N_h} \text{Att}[\mathbf{W}_q^{(i)}, (\mathbf{W}_k^{(i)},\mathbf{W}_v^{(i)},\mathbf{W}_o^{(i)}, \mathbf{X})],
\end{equation*}
where $\mathbf{W}_q,\mathbf{W}_k,\mathbf{W}_v$ and $\mathbf{W}_o$ refer to the query/key/value/output projection matrices, and $\text{Att}(\cdot)$ is an attention function. Following the attention head, the outputs are passed through the feed-forward layer, which consists of intermediate and output-projection layers, parameterized by $\mathbf{W}_{mlp}^i$ and $\mathbf{W}_{mlp}^o$:
\begin{equation*}
\text{FFN}(X)=\text{gelu}(\mathbf{X}\mathbf{W}_{mlp}^{i}) \mathbf{W}_{mlp}^o.
\end{equation*}
Denote by $\mathbf{W}_0 \in \mathbb{R}^{k \times d}$ a pre-trained weight matrix out of $\{ \mathbf{W}_q,\mathbf{W}_k,\mathbf{W}_v, \mathbf{W}_o, \mathbf{W}_{mlp}^i, \mathbf{W}_{mlp}^o\}$. To enforce structured sparsity of the matrix $\mathbf{W}_0$, we propose to multiply it by the learnable \textit{stochastic gates} vectors $\bm{\omega}_{r} \in \{0,1\}^{1 \times k}, \bm{\omega}_{c} \in \{0,1\}^{1 \times d}$ which are trained to converge into the binary representation. For simplicity, we denote by $\bm{\omega}$ a general gate vector applied to rows or columns. To approximate binary values for vector $\bm{\omega}$, we learn a representation $\bm{\mu} \in [-1,1] ^ {1 \times d}$ which is then converted to the approximate Bernoulli variables $\bm{\omega}$, by utilizing a Gaussian-based relaxation of the Bernoulli variables \citep{yamada2020feature,jana2023support}. The relaxation relies on the reparameterization trick \citep{miller2017reducing, figurnov2018implicit} and has been demonstrated to be effective in several applications \citep{lindenbaum2021l0,yang2023multi,naorhybrid}. During the training, the conversion is done by adding a random noise vector $\bm{\epsilon} \in \mathbb{R} ^{1 \times d}$ to the shifted by scalar $0.5$ vector $\bm{\mu}$ and clipping the values by the range of $[0,1]$, 
\begin{equation}
    \bm{\omega}(\bm{\mu}) = \max(0, \min(1, 0.5 + \bm{\mu} + \bm{\epsilon})), 
\end{equation}
where each value in the vector $\bm{\epsilon}$ is drawn from $\mathcal{N}(0, \sigma^2)$ and $\sigma=0.5$ is fixed throughout training. To encourage the model to produce a sparse $\bm{\omega}$ vector, it is trained with the regularization loss term constrained by the given sparsity ratio $s$ and is minimized during the training: 
\begin{equation}
     \mathcal{L}_{\text{sparse}}(\bm{\omega}) = \max (||\bm{\omega}||_0,  s ).
     \label{eq:sparsity}
\end{equation}
Assuming that $\bm{\omega}$ is a Bernoulli variable, we calculate its expected $\ell_0$ norm as follows (the full derivation appears in Appendix \ref{app:sparsity_loss}):
\begin{equation*}
\mathbb{E} ||\bm{\omega}||_0 = \frac{1}{d}\sum_j  \mathbb{P}(\omega_j > 0) = \frac{1}{d}\sum_j    \mathbb{P}(\mu_j + 0.5 + \epsilon_j > 0)= 
\end{equation*}
\begin{equation*}
= \frac{1}{d} \sum_j \left( \frac{1}{2} + \frac{1}{2} \erf \left(\frac{\mu_j + 0.5 }{\sqrt{2} \sigma}\right)\right).
\end{equation*}

When applying the sparsity term in Eq. \ref{eq:sparsity}, the model aims to sparsify the matrix $\mathbf{W}_0$ and retain only the essential parameters needed for the new task.

In addition, we rescale $\bm{\omega}$ by an additional score vector $\mathbf{k}$ in the sparsity loss term. This rescaling is designed to preserve vector outliers in the weight columns and rows, a crucial aspect for maintaining good performance in pruning and quantization tasks  \citep{zhao2024apt, lin2024awq, dettmers2022gpt3}:
\begin{equation}
\mathcal{L}_{\text{sparse}}'(\bm{\omega}) = \max ( ||\mathbf{k} \odot \bm{\omega}||_0, s),
     \label{eq:sparsity2}
\end{equation}
where the $\ell_0$ term is approximated by:
\begin{equation*}
\mathbb{E} ||\mathbf{k} \odot \bm{\omega}||_0 = \frac{1}{d} \sum_j k_j \left( \frac{1}{2} + \frac{1}{2} \erf \left( \frac{\mu_j + 0.5 }{\sqrt{2} \sigma}\right)\right),
\end{equation*}
and
\begin{equation}
    k_j = {e^{-\text{kurt}_j(\bm{O})}}/{\sum_j e^{-\text{kurt}_j(\bm{O})}},
    \label{eq:kurt}
\end{equation}
where $\text{kurt}_j(\cdot)$ in Eq.\ref{eq:kurt} is a Pearson’s kurtosis \citep{decarlo1997meaning} computed for each column $j$ in activations matrix
$\bm{O} =[[\bm{\omega}_r \odot \mathbf{W} \odot \bm{\omega}_c ] \mathbf{X}^T].$
More details can be found in the Appendix \ref{appx:kurt}. 
We require the model to close the gates (sparsify) for rows and columns with lower kurtosis and to attenuate the gate closing for rows and columns with higher kurtosis values.

Finally, assuming a latent representation is obtained by $\mathbf{H}= \mathbf{W}_0\mathbf{X}^T + \mathbf{b}$, our method's forward pass yields:
\begin{equation}
     \mathbf{H} =  [ \bm{\omega}_r \odot \mathbf{W}_0 \odot \bm{\omega}_c] \mathbf{X}^T + \bm{\omega}_r \odot \mathbf{b},
\end{equation}
where $\bm{\omega}_r, \bm{\omega}_c$ are trainable parameters, and $\mathbf{b}$ is a bias parameter. Our method is depicted in Figure \ref{fig:illusration}. At the start of training, we initialize the vectors $\bm{\omega}$  with all elements set to one.
During training, we optimize the parameters \(\bm{\omega}\) that are used to multiply the matrices \(\mathbf{W}_q\), \(\mathbf{W}_k\), \(\mathbf{W}_v\), \(\mathbf{W}_o\), and \(\mathbf{W}_{mlp}\). Additionally, we conduct experiments with an extended version that also trains the \(\mathbf{\Gamma}\) parameters. These parameters are used to assemble the matrices \(\mathbf{W}_A\) and \(\mathbf{W}_B\) for each layer, as proposed by \citep{hu2021lora}.

Our method\footnote{Our code is available on GitHub: \url{https://github.com/jsvir/FineGates}} introduces a simple yet powerful mechanism for training binary gates that explicitly select the most informative rows and columns of the weight matrices. This targeted sparsification not only preserves, but in many cases improves, the accuracy of finetuning tasks compared to existing approaches. At the same time, it delivers substantial parameter reduction in the base model, with only negligible performance loss. In the following section, we present a comprehensive empirical evaluation demonstrating the effectiveness and robustness of our approach.

\section{Theoretical Analysis}

\subsection{FineGates vs LoRA Landscape}
In Section \ref{sec:exps}, we provide empirical evidence of matching or superior performance of our method compared against LoRA, while requiring fewer trainable parameters. We provide the theoretical justification for such performance by comparing the optimization landscapes of LoRA and FineGates.

To that end, we consider a single layer and a smooth loss function (e.g., cross entropy or MSE) that satisfies the Polyak--\L{}ojasiewicz (PL), a mild assumption often used in optimization, which guarantees that gradient descent on $\bW$ converges linearly to a global optimum.

\begin{definition}[Polyak--\L{}ojasiewicz (PL) Condition]
    A function $F(\bW)$ satisfies PL condition with parameter $\beta > 0$ if:
    $\frac{1}{2}||\nabla F(\bW)||^2 \geq \beta (F(\bW)-F^*),$
    where $F^* = \min_{\bW} F(\bW)$.
\end{definition}

\begin{proposition}[FineGates vs LoRA optimization landscape ]
\label{prop:landscape}
Let $\bW_0\in\mathbb{R}^{m\times n}$ to be a weights matrix in a single linear layer trained with a smooth loss function $\mathcal{L}(\bW)$ 
that satisfies the PL condition. Let
\begin{itemize}
\item $W_{\text{gates}}(\bomega_r,\bomega_c) = Diag(\bomega_r) \bW_0 Diag(\bomega_c)$, where $\bomega_r \in\mathbb{R}^m, \bomega_c \in\mathbb{R}^n$ 
are trainable row/column gates.
\item $W_{\text{LoRA}}(\mathbf{A},\mathbf{B}) = \bW_0 + \mathbf{AB}^\top$ with $\mathbf{A}\in\mathbb{R}^{m\times k}, \mathbf{B}\in\mathbb{R}^{n\times k}$. \textbf{Then}
\end{itemize}
\begin{enumerate}
\item $\exists \beta_{\text{gates}} > 0$, s.t. $\forall (\bomega_r, \bomega_c)$, $\mathcal{L}(W_{\text{gates}}(\bomega_r,\bomega_c))$ satisfies PL condition with $\beta_{\text{gates}}$. 
\item $\not \exists \beta_{\text{LoRA}}$, s.t. $\mathcal{L}(W_{\text{LoRA}}(\mathbf{A},\mathbf{B}))$ holds PL condition with $\beta_{\text{LoRA}}$.

\end{enumerate}
\end{proposition}

Compared to LoRA, the FineGates parameterization yields a substantially simpler optimization landscape. Each gate parameter directly scales a single row or column of the base weight matrix, which results in 
gradients that decompose into independent row- and column-wise statistics. This structure aligns well  
with the underlying loss geometry and, under mild assumptions such as smoothness and the PL condition, guarantees linear convergence of gradient-based training. 
In contrast, the LoRA parameterization introduces a bilinear factorization $\mathbf{AB}^\top$, where many 
different choices of $(\mathbf{A},\mathbf{B})$ correspond to the same effective update. This redundancy induces a 
$k^2$-dimensional family of flat directions in the loss surface, manifested as zero eigenvalues of the 
Hessian, and thereby weakens convergence guarantees. As a result, optimization with LoRA can be less 
stable and often requires additional normalization or careful tuning, while gating provides a more 
direct and well-conditioned optimization path. We provide the detailed proof in the Appendix \ref{app:landscape_proof}. In addition to the empirical results reported in Section~\ref{sec:exps}, we perform a separate experiment on the MRPC dataset to further validate our proposition. Figure~\ref{fig:convergence} shows the mean accuracy along with its standard deviation, computed after each training epoch while 
running both methods under identical hyperparameters across ten random initialization seeds.

\begin{figure}[t]
\centering
%\vskip -0.1in
\includegraphics[width=.99\linewidth]{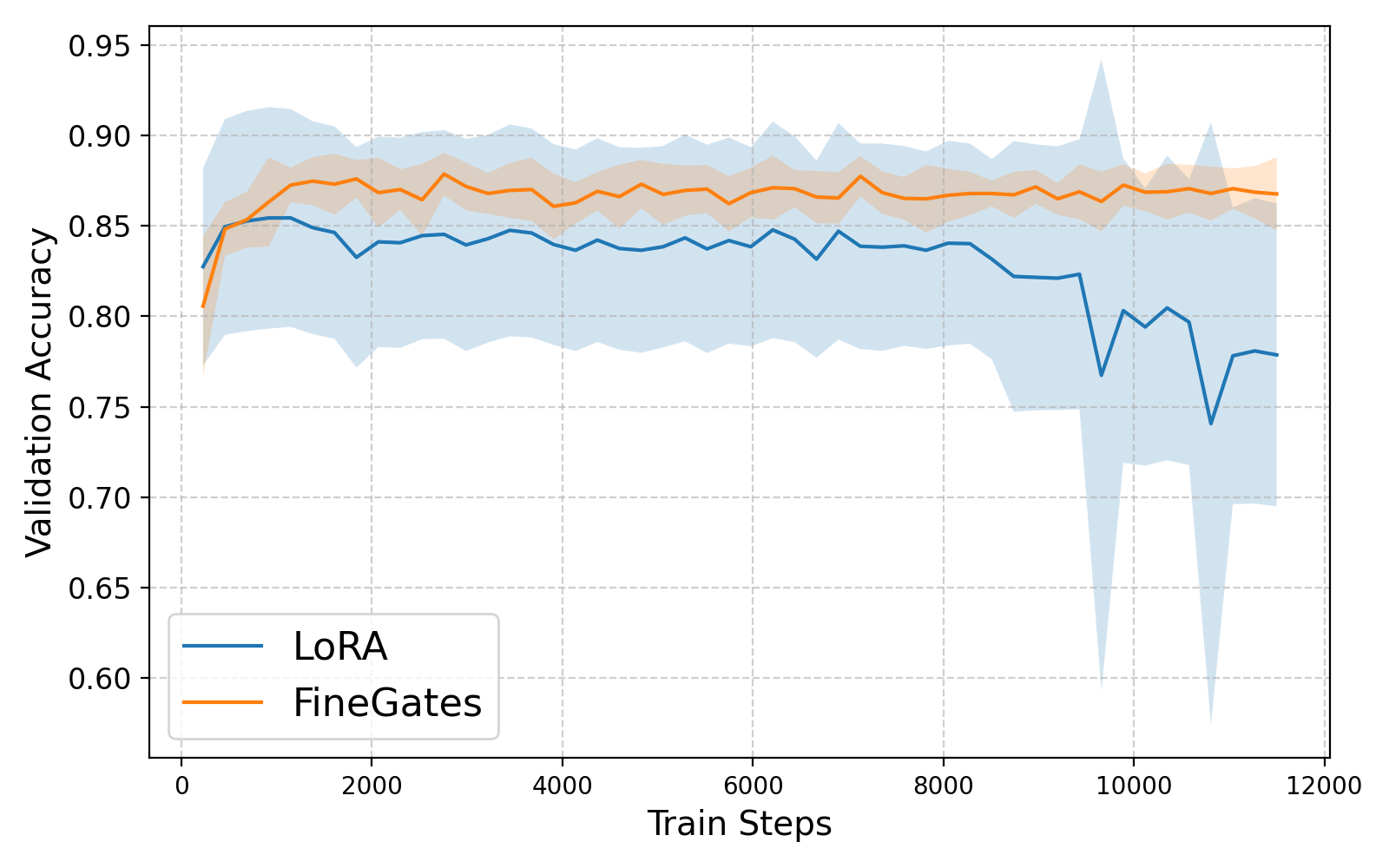}
%\vskip -0.1in
\caption{Validation accuracy trajectories of FineGates and LoRA on the MRPC dataset.} 
\label{fig:convergence}
%\vskip -0.1in
\end{figure}

\subsection{FineGates Convergence}
In this section, we present a proof of convergence for our method. This theoretical justification is necessary because including random noise in the gating mechanism could introduce challenges to the training process and affect convergence. To that end, we start by defining smoothness for functions parameterized by tensors.

\begin{definition}[$L$-continuity] A function $\boldsymbol{h}(\bbw)$ is a  $L$-continues, if for any $\bbw_1$ and $\bbw_2$, $\| \boldsymbol{h}\left(\bbw_1\right)  - \boldsymbol{h}\left(\bbw_2\right)\left\|_F \leq L\right\| \bbw_1-\bbw_2 \|_F$.
\end{definition}

\begin{lemma}[Convergence of FineGates]\label{prop:two-sided-gates}
Consider the FineGates loss function $\mathcal{L}(\bW,\bomega_{r},\bomega_{c})$ in (\ref{def::total_loss}), a network $\Phi(\cdot)$. Suppose that the composition $\mathcal{L} \equiv f \circ\Phi$ is a $L$-continuity, non-convex function, where $f$ is the loss (e.g cross entropy), then $\lim_{t\to\infty}\|\nabla \mathcal{L}(\bW_t,\bomega_{r,t},\bomega_{c,t})\|=0$. In particular, every accumulation point is a first-order stationary point of $\mathcal{L}.$
\end{lemma}

The detailed proof of Proposition \ref{prop:two-sided-gates} can be found in Appendix \ref{app::proof}. The proof establishes the convergence of the loss function $\mathcal{L}(\bW,\bomega_r,\bomega_c)$ under gradient descent. First, it shows that the function's gradient is Lipschitz continuous, meaning that a constant multiple of parameter differences bounds the gradient differences. Then, it verifies that the regularization term is continuously differentiable and bounded, ensuring a smooth optimization process. Using the descent lemma, the proof demonstrates that the loss function decreases monotonically with each gradient step. Since the difference in loss is bounded below, the sum of gradient norms is shown to be finite, implying that the gradients approach zero. Consequently, gradient descent converges to a critical point where no further improvement is possible.

\section{Experiments}
\label{sec:exps}
\subsection{Experimental Setup and Datasets}

We assess the performance of our method on downstream tasks using the RoBERTa-base, RoBERTa-large models \citep{liu2019roberta}, and Llama3.2-1B, and we utilize the widely recognized GLUE benchmark \citep{wang2019superglue}. 

We simulate real-world scenarios where only limited labeled data is available for finetuning tasks due to the challenging nature of collecting accurate ground truth labels. In addition to the full GLUE benchmark, we conduct separate experiments on a limited GLUE dataset with up to 10,000 samples. Accordingly, we use the small-scale datasets (COLA, STSB, MRPC, RTE) in their entirety and select the first 10,000 labeled samples from the larger datasets (MNLI, QQP, QNLI, SST2). These experiments aim to determine whether the proposed method can compress base models with limited training data.

Moreover, we evaluate our method for base-model pruning without a specific target task. For these experiments, we apply our method during pre-training the Llama-1B model on the C4 dataset \citep{raffel2020exploring}. Additionally, our framework is trained with the Llama-2-7B backbone frozen and then evaluated on a zero-shot task. We report the median validation accuracy value for each experiment over five random initialization seeds. The number of trainable parameters (TP) excludes the classifier head, following the same setup as in previous works, e.g., \citet{hu2021lora}. In sparsified pre-training (Section \ref{sec:sparse_pretrain} and general model pruning experiments (Section \ref{sec:sparse_posttrain}), we report perplexity computed on the validation set from C4 
dataset. Additional technical details could be found in the Appendix.

\begin{table*}[t]
\begin{center}
\caption{Evaluation results on the GLUE benchmark demonstrate that our method outperforms other efficient finetuning approaches as well as full finetuning, while also enabling base model \textbf{pruning} and \textbf{inference speedup}.}
\resizebox{.8\linewidth}{!}{
\begin{tabular}{ll|llllllll|l} 
\hline
Method & TP & CoLA & STS-B & MRPC & RTE & SST2 & MNLI & QNLI & QQP & Avg.\\ 
\hline 
\multicolumn{10}{c}{\textbf{RoBERTa-Base}} \\
\hline
Full Finetune & 125M & 63.6 & 90.9 & 90.2 & 80.5 & 94.8 & 87.6 &92.8 & 91.9 & 86.5
\\
Galore \citep{zhao2024galore} & 125M & 60.3 & 90.7 & 92.2& 79.4& 94.0& 87.0& 92.2 &91.1& 85.9 \\
LoRA($r=4$) \citep{hu2021lora} & 0.7M & 64.0 & 90.9 & 89.7 & 83.4 & 94.4 & 87.6 & 92.7 & 91.0 & 86.6
\\
\hline
BitFit \citep{zaken2021bitfit} & 0.11M & 61.8 & 90.8 & \textbf{92.0} & 77.8 & 93.7 & 85.2 & 91.3 & 84.5 & 84.6 \\
VeRA \citep{kopiczko2023vera} & 0.04M & 65.6 & 90.7 & 89.5 & 78.7 & 94.6 & - & 91.8 & - & 85.2 \\
RoCoFT$_{\text{1-Row}}$ \citep{kowsher2024rocoft} & 0.08M & 60.2 & 90.7 & 87.7 &76.6 & 94.1 & 85.2 & 90.7 & 88.5 & 84.2 \\

VeLoRA \citep{miles2024velora} & 0.16M & 64.6 & 90.8& 91.3& 78.0 & 94.4& \textbf{86.3} &92.1 & \textbf{89.9} & 85.9 \\

\rowcolor[HTML]{c9f2c4} FineGates & 0.17M &\textbf{65.7} & \textbf{91.0} & 90.2 & \textbf{83.4} & \textbf{94.7} & 85.8 & \textbf{92.3} & 89.2 & \textbf{86.6 }\\ 
\hline
\multicolumn{10}{c}{\textbf{RoBERTa-Large}} \\
\hline
Full Finetune & 355M & 68.0 & 92.3 & 90.9 & 86.6 & 96.4 & 90.2 & 94.7 & 92.2 & 88.9 \\
LoRA($r=4$) \citep{hu2021lora} & 1.8M & 71.0 &92.3 & 90.7 & 89.5 & 96.4 &90.4 & 94.8 & 91.7 & 89.3
\\
\hline
LoRA-XS \citep{balazy2024lora} & 0.06K &  68.5 &  92.2 & \textbf{91.2} & 89.5 & 96.3 & - & 94.3 & - & 88.7 \\
VeRA \citep{kopiczko2023vera} & 0.06M & 68.0 & 91.7 &  90.9 & 85.9 & 96.1 & - & \textbf{94.4} & - & 87.8 \\
RoCoFT$_{\text{1-Row}}$ \citep{kowsher2024rocoft} & 0.22M & 65.7 & 91.8 & 90.0 & 85.3 & \textbf{96.6} & \textbf{90.7} & 94.2 & \textbf{90.2} & 88.1 \\

\rowcolor[HTML]{c9f2c4} FineGates & 0.4M & \textbf{71.4} & \textbf{92.3} & \textbf{91.2} &\textbf{ 90.2} & 96.0 & 89.1 & 94.1 & 89.4 & \textbf{89.2} \\
\hline
\end{tabular}}
\label{tab:glue_full}
\end{center}
\end{table*} 

\begin{table*}[t]
\begin{center}
\caption{This table compares FineGates with the APT method employing a RoBERTa-Base backbone. In our model, we set the rank to 1, as we are training a single column and a single row of gate vectors. 
}
    \resizebox{.65\linewidth}{!}{
    \begin{tabular}{llllc|ccccc} 
    \hline
    Method & Min $r$ & Max $r$ & TP & Sparsity & MRPC & SST2 & STSB & RTE & CoLA \\ 
    \hline
    APT & 2 & 4 & 0.17M & $\mathbf{40\%}$ & \textbf{86.0 }& 90.6  & \textbf{89.4} & \textbf{67.5} & \textbf{45.9} \\
    \rowcolor[HTML]{c9f2c4} FineGates & 1 & 1 & 0.17M & $\mathbf{40\%}$ &  84.8 & \textbf{91.2} & 88.1 & \textbf{67.5} &  40.6 \\
    \hline 
    \end{tabular}
    }
 \label{tab:sparsity_apt}
\end{center}
\end{table*}

\begin{table}[t]
\centering
\caption{Performance of our method vs. LoRA on GLUE tasks using Llama3.2-1B backbone.}
\resizebox{.99\linewidth}{!}{
\begin{tabular}{lc|ccccc}
\toprule
\textbf{Model} & TP & \textbf{CoLA} & \textbf{STS-B} & \textbf{MRPC} & \textbf{RTE} & \textbf{SST-2} \\
\midrule
LoRA($r=4$) & 2.2M & 66.1 & 90.1 & 87.7 & 83.7 & 95.9 \\
\rowcolor[HTML]{c9f2c4} FineGates & 0.9M  & 63.6 & 91.2 & 88.5 & 84.5 & 93.5 \\
\bottomrule
\end{tabular}}
\label{tab:llama_results}
\label{tab:llama1b_glue}
\end{table}

\begin{table*}[t]

\begin{center}
\caption{Finetuning accuracy on limited GLUE benchmark datasets. We present the accuracy results achieved by FineGates with various sparsity constraints: $s \geq 0\%$, $s \geq 10\%$, and $s \geq 20\%$. The number of removed parameters is shown in {\color{olive}{olive}}, and the relative change in accuracy is depicted in {\color{gray}{gray}} compared to full finetuning.
}

\resizebox{1.\linewidth}{!}{
\begin{tabular}{ll|llll|llll} 
\hline
\multicolumn{2}{c}{} & \multicolumn{4}{c}{Small} & \multicolumn{4}{c}{Limited Large}\\
Method & TP $\downarrow$ & CoLA & STS-B & MRPC & RTE & SST2 & MNLI & QNLI & QQP \\ 
\hline 
\multicolumn{10}{c}{\textbf{RoBERTa-Base}} \\
\hline
Full Finetune & 125M & 63.6 & 90.9 & 90.2 & 80.5 & 92.8 & 81.4 & 87.7 & 85.2  \\
FineGates, $s \geq 0\%$ & 0.17M & 65.7 & 91.0 & 90.2 & 83.4 & 94.0 & 81.3 & 89.1 & 84.9 \\
FineGates, $s \geq 10\%$ \color{olive}{(-12M)} & 0.17M & 65.2 ({\color{gray} +1.6\%}) & 90.7({\color{gray} -0.2\%}) & 89.2 ({\color{gray} -1\%}) & 81.3 ({\color{gray} +0.8\%})& 94.0 ({\color{gray} +1.2\%})& 80.6 ({\color{gray} -0.8\%}) & 88.8 ({\color{gray} +1.1\%}) & 84.3 ({\color{gray} -0.9\%})\\ 
FineGates, $s \geq20\%$ \color{olive}{(-25M)} & 0.17M & 61.4 ({\color{gray} -2.2\%}) & 90.5 ({\color{gray} -0.4\%})& 87.7 ({\color{gray} -2.5\%})& 78.0 ({\color{gray} -2.5\%})& 93.8 ({\color{gray} +1\%})& 79.4 ({\color{gray} -2\%})& 87.6 ({\color{gray} -0.1\%})& 83.7 ({\color{gray} -1.5\%})\\
\hline
\multicolumn{10}{c}{\textbf{RoBERTa-Large}} \\
\hline
Full Finetune & 355M & 68.0 & 92.3 & 90.9 & 86.6 & 93.5 & 85.9 & 92.4 & 85.8 \\
FineGates, $s \geq 0\%$ & 0.4M & 71.4 & 92.3  & 91.2 & 90.2 & 96.2 & 86.2 & 92.4 & 86.1 \\
FineGates, $s \geq10\%$ \color{olive}{(-35M)} & 0.4M & 69.5 ({\color{gray} +1.5\%})& 91.8 ({\color{gray} -0.5\%})& 90.4 ({\color{gray} -0.5\%})& 89.2 ({\color{gray} +2.6\%})& 96.1 ({\color{gray}  +2.6\%})& 85.0 ({\color{gray} -0.9\%})& 92.1 ({\color{gray} -0.3\%})& 86.1 ({\color{gray} +0.3\%})\\ 
FineGates, $s \geq20\%$ \color{olive}{(-70M)} & 0.4M & 68.1 ({\color{gray} +0.1\%})& 91.3 ({\color{gray} -1\%})& 90.2 ({\color{gray} -0.7\%})& 87.0 ({\color{gray} +0.4\%})& 95.3 ({\color{gray} +1.8\%})& 85.0 ({\color{gray} -0.9\%})& 92.0 ({\color{gray} -0.4\%})& 85.2 ({\color{gray} -0.6\%})\\ 
\hline
\end{tabular}}
%\vskip -0.1in

\label{tab:glue_combined}
\end{center}
%\vskip -0.1in
\end{table*} 
\subsection{Baselines}
We compare our method to full finetuning, in which the model is initialized to the pre-trained weights and biases, and all parameters undergo gradient updates, and several recently proposed efficient finetuning methods: LoRA \citep{hu2021lora} with rank$=4$, BitFit \citep{zaken2021bitfit}, VeRA \citep{kopiczko2023vera}, LoRA-XS \citep{balazy2024lora}, RoCoFT$_{\text{1-Row}}$ \citep{kowsher2024rocoft}, APT \citep{zhao2024apt}, and VeLoRA \citep{miles2024velora}. We provide additional descriptions for the baseline methods in the Appendix \ref{app:baselines}.

\subsection{Accuracy Results}

We present the accuracy results in Table \ref{tab:glue_full}. We report the overall (matched and mismatched) accuracy for MNLI, Matthew’s correlation for CoLA, Pearson correlation for STS-B, and accuracy for other tasks. 
From Table \ref{tab:glue_full}, it can be seen that our model is comparable to LoRA and full finetuning on average when applied to the RoBERTa-Base base model. In addition, our model outperforms full finetuning on average when applied to the RoBERTa-Large base model. Moreover, FineGates outperforms other efficient finetuning methods with both backbones. Table~\ref{tab:llama_results} presents a comparison between our method and LoRA using the modern Llama3.2-1B backbone. Despite utilizing fewer trainable parameters, our approach achieves higher accuracy than LoRA on most tasks.

Furthermore, our approach not only reduces the trainable parameter count but also compresses the base model itself, resulting in a $10-20\%$ reduction in its parameter count while maintaining an insignificant loss in accuracy (as shown in the last two rows of Table \ref{tab:glue_combined}). This highlights the efficiency and effectiveness of our method in balancing compression and performance. 

\subsection{Sparsification Results}

We present the sparsification results of FineGates in Figure \ref{fig:sparsity}. We report accuracy measurements for each sparsity level. It could be seen that our method can remove up to $20\%$ of parameters without significant loss in Matthews correlation for the CoLA dataset, up to $40\%$ of parameters trained on the SST2 dataset with only a loss of $4\%$ in accuracy, and up to $40\%$ for the STSB dataset with a loss of only $~3\%$ in Pearson Correlation metric. 

\subsection{Comparison to the APT method}
\label{sec:apt}
To compare our method with the recently proposed APT model, we conducted experiments on the MRPC, STSB, SST2, and RTE datasets. We obtained results for the APT model by running it with LoRA rank values ranging from 2 to 4, using hyperparameters similar to those in our method. We fixed the target sparsity at 40\% for both methods. The results are presented in Table \ref{tab:sparsity_apt}. FineGates achieves performance comparable to the APT model, without the computationally intensive pruning used by APT. We would also like to emphasize the simplicity of our implementation, which does not require the efficient search tools and additional dependencies needed by the APT method. 

\subsection{Inference Speedup}
\label{sec:inference_time}
We evaluate the inference time of the compressed model after each training epoch. For every epoch, validation wall-clock time is measured 10 times per run. To account for variability, training is repeated with 3 different random seeds, yielding a total of 30 inference time measurements per epoch. All validation measurements are conducted in CPU-only environments to isolate the benefits of structured pruning. Importantly, our reductions are obtained before any GPU-specific optimizations or quantization methods, underscoring the intrinsic efficiency gains of our approach, independent of hardware accelerators or low-level optimizations. In this practical setting, FineGates reduces the inference time of the adopted Llama3.2-1B model on the MRPC task by up to 25\%, with only a 4\% loss in accuracy.  More details are in Appendix \ref{app:inference_time}.

\subsection{General Model Sparsification}
\label{sec:sparse_posttrain}
Inspired by the recently proposed method for training unstructured and semi-structured pruning masks for the base model, denoted as MaskLLM \citep{fang2024maskllm}, we conduct an additional experiment to evaluate FineGates in this scenario. We train our method using the Llama-2-7B backbone with a limited number of training tokens (TT) and evaluate its performance on the C4 validation set.
Notably, our method requires only \textbf{2M} trainable parameters (TP) while MaskLLM requires an additional trainable matrix mask for each base model weight matrix. Moreover, MaskLLM proposes a semi-structured sparsity, which is beneficial for GPU execution, whereas our method achieves structured sparsity more suited to CPU inference. In this experiment, we train FineGates without bias and freeze the parameters of the original base model. The perplexity results are shown in Table \ref{tab:sparse_postrain}. The evaluation of all methods is done with a maximal sequence length of 256, which is different from the length of 4096, reported by \citet{fang2024maskllm}.
\begin{table}[t]
\centering
\caption{Validation perplexity of the sparsified Llama-2-7B.
}
\resizebox{.7\linewidth}{!}{
\begin{tabular}{lc|cccl}
\toprule

\textbf{Model} & \textbf{C4} & \textbf{TP} & \textbf{TT} & \textbf{Sparsity}\\
\midrule
Pretrained               & 7.76 &  6,738M & 2T & N/A\\
\midrule
MaskLLM          & 11.15 & 6,738M & 2B & 2:4  \\
\rowcolor[HTML]{c9f2c4} FineGates   & 11.75 & \textbf{2M} &  \textbf{82M} & \textbf{21\%}\\
\bottomrule
\end{tabular}}
\label{tab:sparse_postrain}
\end{table}

\subsection{Sparse Model Pre-training }
\label{sec:sparse_pretrain}

We also evaluate our method in the pre-training scenario, where the base model is trained from scratch with the gates. We follow the same setup as suggested by \citet{zhao2024galore} and train the Llama-1B base model with 1,340M parameters on the C4 dataset. We conduct a full finetuning of the base model (FFT) and its sparsified version (FineGates + FFT) for a limited number of training tokens (TT) and present the perplexity (PPL), number of inference parameters (IP), and number of added training parameters (ATP) in Table \ref{tab:sparse_pretrain}. By adding only $0.09\%$ parameters to the base model during pretraining, our method enables faster convergence of sparsified weights and structured pruning of 44\% (584M parameters) of the model's weights.  

\begin{table}[t]
\centering
\caption{Pre-training Llama-1B base model from scratch with FineGates under limited training budget. 
}
\resizebox{.8\linewidth}{!}{
\begin{tabular}{lrrrr}
\toprule
\textbf{Model} & \textbf{PPL} & \textbf{TT (B)} & \textbf{IP (M)} & \textbf{ATP (M)} \\
\midrule
FFT               & 342.75 & 1.9   & 1,340 & 0    \\
\rowcolor[HTML]{c9f2c4} FineGates + FFT   & \textbf{24.87}  & 1.9   & \textbf{756}   & 1.21 \\

\midrule
{\color{gray}FineGates + FFT } & {\color{gray}21.41 } & {\color{gray}4}   & {\color{gray}690}   & {\color{gray}1.21} \\
\bottomrule
\end{tabular}}
\label{tab:sparse_pretrain}
\end{table}

\section{Conclusion}

In this work, we propose \emph{FineGates}, a sparsification-based finetuning method for foundation language models (LMs). By learning binary stochastic gating vectors on weight matrices, FineGates introduces structured sparsity that can remove up to 40\% of parameters in attention layers while preserving accuracy. This approach offers a more efficient alternative to conventional low-rank adaptation methods, reducing inference costs without sacrificing expressiveness. Experiments on the GLUE benchmark show that FineGates matches or outperforms existing finetuning approaches, and we further provide a theoretical analysis of its convergence and optimization landscape. Beyond task-specific adaptation, we also demonstrate the effectiveness of FineGates as a pruning strategy during pre-training, enabling faster convergence and improved inference efficiency. Overall, our results highlight structured sparsification as a powerful mechanism for efficient and scalable LM adaptation. 

\bibliography{refs}

@inproceedings{
refael2025adarankgrad,
title={AdaRankGrad: Adaptive Gradient Rank and Moments for Memory-Efficient {LLM}s Training and Fine-Tuning},
author={Yehonathan Refael and Jonathan Svirsky and Boris Shustin and Wasim Huleihel and Ofir Lindenbaum},
booktitle={The Thirteenth International Conference on Learning Representations},
year={2025},
}

@article{zhu2024apollo,
  title={APOLLO: SGD-like Memory, AdamW-level Performance},
  author={Zhu, Hanqing and Zhang, Zhenyu and Cong, Wenyan and Liu, Xi and Park, Sem and Chandra, Vikas and Long, Bo and Pan, David Z and Wang, Zhangyang and Lee, Jinwon},
  journal={arXiv preprint arXiv:2412.05270},
  year={2024}
}

@article{chen2024fira,
  title={Fira: Can We Achieve Full-rank Training of LLMs Under Low-rank Constraint?},
  author={Chen, Xi and Feng, Kaituo and Li, Changsheng and Lai, Xunhao and Yue, Xiangyu and Yuan, Ye and Wang, Guoren},
  journal={arXiv preprint arXiv:2410.01623},
  year={2024}
}

@article{Hao2024FloraLA,
  title={Flora: Low-Rank Adapters Are Secretly Gradient Compressors},
  author={Yongchang Hao and Yanshuai Cao and Lili Mou},
  journal={ArXiv},
  year={2024},
  volume={abs/2402.03293},
  url={https://api.semanticscholar.org/CorpusID:267412117}
}

@article{zhang2024adam,
  title={Adam-mini: Use fewer learning rates to gain more},
  author={Zhang, Yushun and Chen, Congliang and Li, Ziniu and Ding, Tian and Wu, Chenwei and Ye, Yinyu and Luo, Zhi-Quan and Sun, Ruoyu},
  journal={arXiv preprint arXiv:2406.16793},
  year={2024}
}

@misc{robert2025ldadama,
      title={LDAdam: Adaptive Optimization from Low-Dimensional Gradient Statistics}, 
      author={Thomas Robert and Mher Safaryan and Ionut-Vlad Modoranu and Dan Alistarh},
      year={2025},
      eprint={2410.16103},
      archivePrefix={arXiv},
      primaryClass={cs.LG},
      url={https://arxiv.org/abs/2410.16103}, 
}

@article{naorhybrid,
  title={Hybrid Autoencoders for Tabular Data: Leveraging Model-Based Augmentation in Low-Label Settings},
  author={Naor, Erel and Lindenbaum, Ofir},
    journal={The Thirty-ninth Annual Conference on Neural Information Processing Systems},
    year={2025}
}

@misc{loeschcke2024,
      title={LoQT: Low-Rank Adapters for Quantized Pretraining}, 
      author={Sebastian Loeschcke and Mads Toftrup and Michael J. Kastoryano and Serge Belongie and Vésteinn Snæbjarnarson},
      year={2024},
      eprint={2405.16528},
      archivePrefix={arXiv},
      primaryClass={cs.LG},
      url={https://arxiv.org/abs/2405.16528}, 
}

@misc{refael2025sumo,
      title={SUMO: Subspace-Aware Moment-Orthogonalization for Accelerating Memory-Efficient LLM Training}, 
      author={Yehonathan Refael and Guy Smorodinsky and Tom Tirer and Ofir Lindenbaum},
      year={2025},
      eprint={2505.24749},
      archivePrefix={arXiv},
      primaryClass={cs.LG},
      url={https://arxiv.org/abs/2505.24749}, 
}

@misc{refael2025lorenza,
      title={LORENZA: Enhancing Generalization in Low-Rank Gradient LLM Training via Efficient Zeroth-Order Adaptive SAM}, 
      author={Yehonathan Refael and Iftach Arbel and Ofir Lindenbaum and Tom Tirer},
      year={2025},
      eprint={2502.19571},
      archivePrefix={arXiv},
      primaryClass={cs.LG},
      url={https://arxiv.org/abs/2502.19571}, 
}

@misc{refael2024learningklevelstructuredsparse,
      title={Learning k-Level Structured Sparse Neural Networks Using Group Envelope Regularization}, 
      author={Yehonathan Refael and Iftach Arbel and Wasim Huleihel},
      year={2024},
      eprint={2212.12921},
      archivePrefix={arXiv},
      primaryClass={cs.LG},
      url={https://arxiv.org/abs/2212.12921}, 
}

@inproceedings{yang2023multi,
  title={Multi-modal differentiable unsupervised feature selection},
  author={Yang, Junchen and Lindenbaum, Ofir and Kluger, Yuval and Jaffe, Ariel},
  booktitle={Uncertainty in Artificial Intelligence},
  pages={2400--2410},
  year={2023},
  organization={PMLR}
}

@article{rozner2024knowledge,
  title={Knowledge Editing in Language Models via Adapted Direct Preference Optimization},
  author={Rozner, Amit and Battash, Barak and Wolf, Lior and Lindenbaum, Ofir},
  journal={arXiv preprint arXiv:2406.09920},
  year={2024}
}

@article{hu2021lora,
  title={Lora: Low-rank adaptation of large language models},
  author={Hu, Edward J and Shen, Yelong and Wallis, Phillip and Allen-Zhu, Zeyuan and Li, Yuanzhi and Wang, Shean and Wang, Lu and Chen, Weizhu},
  journal={arXiv preprint arXiv:2106.09685},
  year={2021}
}

@article{zhao2024galore,
  title={Galore: Memory-efficient llm training by gradient low-rank projection},
  author={Zhao, Jiawei and Zhang, Zhenyu and Chen, Beidi and Wang, Zhangyang and Anandkumar, Anima and Tian, Yuandong},
  journal={arXiv preprint arXiv:2403.03507},
  year={2024}
}

@article{vaswani2017attention,
  title={Attention is all you need},
  author={Vaswani, A},
  journal={Advances in Neural Information Processing Systems},
  year={2017}
}

@article{wen2016learning,
  title={Learning structured sparsity in deep neural networks},
  author={Wen, Wei and Wu, Chunpeng and Wang, Yandan and Chen, Yiran and Li, Hai},
  journal={Advances in neural information processing systems},
  volume={29},
  year={2016}
}

@inproceedings{yamada2020feature,
  title={Feature selection using stochastic gates},
  author={Yamada, Yutaro and Lindenbaum, Ofir and Negahban, Sahand and Kluger, Yuval},
  booktitle={International Conference on Machine Learning},
  pages={10648--10659},
  year={2020},
  organization={PMLR}
}

@article{miller2017reducing,
  title={Reducing reparameterization gradient variance},
  author={Miller, Andrew and Foti, Nick and D'Amour, Alexander and Adams, Ryan P},
  journal={Advances in Neural Information Processing Systems},
  volume={30},
  year={2017}
}

@article{figurnov2018implicit,
  title={Implicit reparameterization gradients},
  author={Figurnov, Mikhail and Mohamed, Shakir and Mnih, Andriy},
  journal={Advances in neural information processing systems},
  volume={31},
  year={2018}
}

@article{jana2023support,
  title={Support recovery with Projected Stochastic Gates: Theory and application for linear models},
  author={Jana, Soham and Li, Henry and Yamada, Yutaro and Lindenbaum, Ofir},
  journal={Signal Processing},
  volume={213},
  pages={109193},
  year={2023},
  publisher={Elsevier}
}

@inproceedings{lindenbaum2021l0,
  title={L0-sparse canonical correlation analysis},
  author={Lindenbaum, Ofir and Salhov, Moshe and Averbuch, Amir and Kluger, Yuval},
  booktitle={International Conference on Learning Representations},
  year={2021}
}

@article{liu2019roberta,
  title={Roberta: A robustly optimized bert pretraining approach},
  author={Liu, Y},
  journal={arXiv preprint arXiv:1907.11692},
  year={2019}
}

@article{wang2019superglue,
  title={Superglue: A stickier benchmark for general-purpose language understanding systems},
  author={Wang, Alex and Pruksachatkun, Yada and Nangia, Nikita and Singh, Amanpreet and Michael, Julian and Hill, Felix and Levy, Omer and Bowman, Samuel},
  journal={Advances in neural information processing systems},
  volume={32},
  year={2019}
}

@article{loshchilov2017decoupled,
  title={Decoupled weight decay regularization},
  author={Loshchilov, Ilya and Hutter, Frank},
  journal={arXiv preprint arXiv:1711.05101},
  year={2017}
}

@article{balazy2024lora,
  title={LoRA-XS: Low-Rank Adaptation with Extremely Small Number of Parameters},
  author={Ba{\l}azy, Klaudia and Banaei, Mohammadreza and Aberer, Karl and Tabor, Jacek},
  journal={arXiv preprint arXiv:2405.17604},
  year={2024}
}

@article{lin2024nora,
  title={NoRA: Nested Low-Rank Adaptation for Efficient Fine-Tuning Large Models},
  author={Lin, Cheng and Li, Lujun and Li, Dezhi and Zou, Jie and Luo, Wenhan and Xue, Wei and Guo, Yike},
  journal={arXiv preprint arXiv:2408.10280},
  year={2024}
}

@article{li2022parameter,
  title={Parameter-efficient sparsity for large language models fine-tuning},
  author={Li, Yuchao and Luo, Fuli and Tan, Chuanqi and Wang, Mengdi and Huang, Songfang and Li, Shen and Bai, Junjie},
  journal={arXiv preprint arXiv:2205.11005},
  year={2022}
}

@article{xu2023qa,
  title={Qa-lora: Quantization-aware low-rank adaptation of large language models},
  author={Xu, Yuhui and Xie, Lingxi and Gu, Xiaotao and Chen, Xin and Chang, Heng and Zhang, Hengheng and Chen, Zhensu and Zhang, Xiaopeng and Tian, Qi},
  journal={arXiv preprint arXiv:2309.14717},
  year={2023}
}

@article{chavan2023one,
  title={One-for-all: Generalized lora for parameter-efficient fine-tuning},
  author={Chavan, Arnav and Liu, Zhuang and Gupta, Deepak and Xing, Eric and Shen, Zhiqiang},
  journal={arXiv preprint arXiv:2306.07967},
  year={2023}
}

@article{zhang2023lora,
  title={Lora-fa: Memory-efficient low-rank adaptation for large language models fine-tuning},
  author={Zhang, Longteng and Zhang, Lin and Shi, Shaohuai and Chu, Xiaowen and Li, Bo},
  journal={arXiv preprint arXiv:2308.03303},
  year={2023}
}

@article{kopiczko2023vera,
  title={Vera: Vector-based random matrix adaptation},
  author={Kopiczko, Dawid Jan and Blankevoort, Tijmen and Asano, Yuki Markus},
  journal={arXiv preprint arXiv:2310.11454},
  year={2023}
}

@article{lin2024awq,
  title={AWQ: Activation-aware Weight Quantization for On-Device LLM Compression and Acceleration},
  author={Lin, Ji and Tang, Jiaming and Tang, Haotian and Yang, Shang and Chen, Wei-Ming and Wang, Wei-Chen and Xiao, Guangxuan and Dang, Xingyu and Gan, Chuang and Han, Song},
  journal={Proceedings of Machine Learning and Systems},
  volume={6},
  pages={87--100},
  year={2024}
}

@article{bondarenko2024low,
  title={Low-Rank Quantization-Aware Training for LLMs},
  author={Bondarenko, Yelysei and Del Chiaro, Riccardo and Nagel, Markus},
  journal={arXiv preprint arXiv:2406.06385},
  year={2024}
}

@article{ma2023llm,
  title={Llm-pruner: On the structural pruning of large language models},
  author={Ma, Xinyin and Fang, Gongfan and Wang, Xinchao},
  journal={Advances in neural information processing systems},
  volume={36},
  pages={21702--21720},
  year={2023}
}

@article{xia2022structured,
  title={Structured pruning learns compact and accurate models},
  author={Xia, Mengzhou and Zhong, Zexuan and Chen, Danqi},
  journal={arXiv preprint arXiv:2204.00408},
  year={2022}
}

@article{hinton2015distilling,
  title={Distilling the Knowledge in a Neural Network},
  author={Hinton, Geoffrey},
  journal={arXiv preprint arXiv:1503.02531},
  year={2015}
}

@inproceedings{houlsby2019parameter,
  title={Parameter-efficient transfer learning for NLP},
  author={Houlsby, Neil and Giurgiu, Andrei and Jastrzebski, Stanislaw and Morrone, Bruna and De Laroussilhe, Quentin and Gesmundo, Andrea and Attariyan, Mona and Gelly, Sylvain},
  booktitle={International conference on machine learning},
  pages={2790--2799},
  year={2019},
  organization={PMLR}
}

@article{pfeiffer2020adapterfusion,
  title={Adapterfusion: Non-destructive task composition for transfer learning},
  author={Pfeiffer, Jonas and Kamath, Aishwarya and R{\"u}ckl{\'e}, Andreas and Cho, Kyunghyun and Gurevych, Iryna},
  journal={arXiv preprint arXiv:2005.00247},
  year={2020}
}

@article{lester2021power,
  title={The power of scale for parameter-efficient prompt tuning},
  author={Lester, Brian and Al-Rfou, Rami and Constant, Noah},
  journal={arXiv preprint arXiv:2104.08691},
  year={2021}
}

@article{li2021prefix,
  title={Prefix-tuning: Optimizing continuous prompts for generation},
  author={Li, Xiang Lisa and Liang, Percy},
  journal={arXiv preprint arXiv:2101.00190},
  year={2021}
}

@article{he2022sparseadapter,
  title={Sparseadapter: An easy approach for improving the parameter-efficiency of adapters},
  author={He, Shwai and Ding, Liang and Dong, Daize and Zhang, Miao and Tao, Dacheng},
  journal={arXiv preprint arXiv:2210.04284},
  year={2022}
}

@article{zhang2023adalora,
  title={AdaLoRA: Adaptive budget allocation for parameter-efficient fine-tuning},
  author={Zhang, Qingru and Chen, Minshuo and Bukharin, Alexander and Karampatziakis, Nikos and He, Pengcheng and Cheng, Yu and Chen, Weizhu and Zhao, Tuo},
  journal={arXiv preprint arXiv:2303.10512},
  year={2023}
}

@article{zhao2024apt,
  title={Apt: Adaptive pruning and tuning pretrained language models for efficient training and inference},
  author={Zhao, Bowen and Hajishirzi, Hannaneh and Cao, Qingqing},
  journal={arXiv preprint arXiv:2401.12200},
  year={2024}
}

@article{sanh2020movement,
  title={Movement pruning: Adaptive sparsity by fine-tuning},
  author={Sanh, Victor and Wolf, Thomas and Rush, Alexander},
  journal={Advances in neural information processing systems},
  volume={33},
  pages={20378--20389},
  year={2020}
}

@inproceedings{frantar2023sparsegpt,
  title={Sparsegpt: Massive language models can be accurately pruned in one-shot},
  author={Frantar, Elias and Alistarh, Dan},
  booktitle={International Conference on Machine Learning},
  pages={10323--10337},
  year={2023},
  organization={PMLR}
}

@article{refael2024adarankgrad,
  title={AdaRankGrad: Adaptive Gradient-Rank and Moments for Memory-Efficient LLMs Training and Fine-Tuning},
  author={Refael, Yehonathan and Svirsky, Jonathan and Shustin, Boris and Huleihel, Wasim and Lindenbaum, Ofir},
  journal={arXiv preprint arXiv:2410.17881},
  year={2024}
}

@article{dettmers2022gpt3,
  title={Gpt3. int8 (): 8-bit matrix multiplication for transformers at scale},
  author={Dettmers, Tim and Lewis, Mike and Belkada, Younes and Zettlemoyer, Luke},
  journal={Advances in Neural Information Processing Systems},
  volume={35},
  pages={30318--30332},
  year={2022}
}

@article{zaken2021bitfit,
  title={Bitfit: Simple parameter-efficient fine-tuning for transformer-based masked language-models},
  author={Zaken, Elad Ben and Ravfogel, Shauli and Goldberg, Yoav},
  journal={arXiv preprint arXiv:2106.10199},
  year={2021}
}

@article{kowsher2024rocoft,
  title={RoCoFT: Efficient Finetuning of Large Language Models with Row-Column Updates},
  author={Kowsher, Md and Esmaeilbeig, Tara and Yu, Chun-Nam and Soltanalian, Mojtaba and Yousefi, Niloofar},
  journal={arXiv preprint arXiv:2410.10075},
  year={2024}
}

@article{decarlo1997meaning,
  title={On the meaning and use of kurtosis.},
  author={DeCarlo, Lawrence T},
  journal={Psychological methods},
  volume={2},
  number={3},
  pages={292},
  year={1997},
  publisher={American Psychological Association}
}

@article{gordon2020compressing,
  title={Compressing bert: Studying the effects of weight pruning on transfer learning},
  author={Gordon, Mitchell A and Duh, Kevin and Andrews, Nicholas},
  journal={arXiv preprint arXiv:2002.08307},
  year={2020}
}

@article{fang2024maskllm,
  title={Maskllm: Learnable semi-structured sparsity for large language models},
  author={Fang, Gongfan and Yin, Hongxu and Muralidharan, Saurav and Heinrich, Greg and Pool, Jeff and Kautz, Jan and Molchanov, Pavlo and Wang, Xinchao},
  journal={Advances in Neural Information Processing Systems},
  volume={37},
  pages={7736--7758},
  year={2024}
}

@article{raffel2020exploring,
  title={Exploring the limits of transfer learning with a unified text-to-text transformer},
  author={Raffel, Colin and Shazeer, Noam and Roberts, Adam and Lee, Katherine and Narang, Sharan and Matena, Michael and Zhou, Yanqi and Li, Wei and Liu, Peter J},
  journal={Journal of machine learning research},
  volume={21},
  number={140},
  pages={1--67},
  year={2020}
}

@article{miles2024velora,
  title={Velora: Memory efficient training using rank-1 sub-token projections},
  author={Miles, Roy and Reddy, Pradyumna and Elezi, Ismail and Deng, Jiankang},
  journal={Advances in Neural Information Processing Systems},
  volume={37},
  pages={42292--42310},
  year={2024}
}

@inproceedings{ceritli2024study,
  title={A study of parameter efficient fine-tuning by learning to efficiently fine-tune},
  author={Ceritli, Taha and Ozkan, Savas and Min, Jeongwon and Noh, Eunchung and Min, Cho and Ozay, Mete},
  booktitle={Findings of the Association for Computational Linguistics: EMNLP 2024},
  pages={15819--15836},
  year={2024}
}

@article{wang2024lora,
  title={Lora-ga: Low-rank adaptation with gradient approximation},
  author={Wang, Shaowen and Yu, Linxi and Li, Jian},
  journal={Advances in Neural Information Processing Systems},
  volume={37},
  pages={54905--54931},
  year={2024}
}

@inproceedings{lei2024fast,
  title={Fast randomized low-rank adaptation of pre-trained language models with pac regularization},
  author={Lei, Zijian and Qian, Dong and Cheung, William},
  booktitle={Findings of the Association for Computational Linguistics ACL 2024},
  pages={5236--5249},
  year={2024}
}

@inproceedings{zhang-etal-2025-parameter,
    title = "Parameter-Efficient Fine-Tuning of Large Language Models via Deconvolution in Subspace",
    author = "Zhang, Jia-Chen  and
      Xiong, Yu-Jie  and
      Xia, Chun-Ming  and
      Zhu, Dong-Hai  and
      Qiu, Xi-He",
    editor = "Rambow, Owen  and
      Wanner, Leo  and
      Apidianaki, Marianna  and
      Al-Khalifa, Hend  and
      Eugenio, Barbara Di  and
      Schockaert, Steven",
    booktitle = "Proceedings of the 31st International Conference on Computational Linguistics",
    month = jan,
    year = "2025",
    address = "Abu Dhabi, UAE",
    publisher = "Association for Computational Linguistics",
    url = "https://aclanthology.org/2025.coling-main.265/",
    pages = "3924--3935"
}

@article{salas1999gershgorin,
  title={Gershgorin's theorem for matrices of operators},
  author={Salas, Hector N},
  journal={Linear algebra and its applications},
  volume={291},
  number={1-3},
  pages={15--36},
  year={1999},
  publisher={Elsevier}
}

\appendix
\onecolumn
\aistatstitle{Train Less, Infer Faster: Efficient Model Finetuning and Compression \\ via Structured Sparsity: Supplementary Materials}
\thispagestyle{empty}
\section{Expected $\ell_0$ norm}
\label{app:sparsity_loss}

\begin{equation*}
\mathbb{E} ||\bm{\omega}||_0 = \frac{1}{d}\sum_j  \mathbb{P}(\omega_j > 0) = \frac{1}{d}\sum_j    \mathbb{P}(\mu_j + 0.5 + \epsilon_j > 0)= 
\end{equation*}
\begin{equation*}
=\frac{1}{d}\sum_j    (1 - \mathbb{P}(\mu_j + 0.5 + \epsilon_j \leq 0)) = 
\end{equation*}
\begin{equation*}
=\frac{1}{d}\sum_j  \left(1 - \Phi \left(\frac{-\mu_j - 0.5}{\sigma} \right) \right) = 
\end{equation*}
\begin{equation*}
 =\frac{1}{d}\sum_j 1 - \frac{1}{2} \left(1 + \erf \left(-\frac{\mu_j + 0.5 }{\sqrt{2} \sigma}\right)\right) = 
 \end{equation*}
\begin{equation*}
= \frac{1}{d} \sum_j \left( \frac{1}{2} + \frac{1}{2} \erf \left(\frac{\mu_j + 0.5 }{\sqrt{2} \sigma}\right)\right).
\end{equation*}

\section{Proof of Proposition \ref{prop:landscape}}
\label{app:landscape_proof}

\paragraph{Setup and notation.}
Let $L:\mathbb{R}^{m\times n}\to\mathbb{R}_+$ be differentiable and satisfy the Polyak–Łojasiewicz (PL) inequality on a set $\mathcal{W}\subseteq\mathbb{R}^{m\times n}$ with parameter $\beta>0$:
\begin{equation}\label{eq:PL-in-W}
\frac{1}{2}\,\|\nabla_{\bW} L(\bW)\|_F^2 \;\ge\; \beta\big(L(\bW)-L^\star\big)\qquad \forall\,\bW\in\mathcal{W},
\end{equation}
where $L^\star=\inf_{\bW\in\mathcal{W}}L(\bW)$.
Write $\mathrm{vec}(\cdot)$ for vectorization and use the Euclidean/Frobenius norms throughout.

\paragraph{Two-sided gates parametrization.}
Fix a base matrix $\bW_0\in\mathbb{R}^{m\times n}$. For row/column gates $\bomega_r\in\mathbb{R}^m$, $\bomega_c\in\mathbb{R}^n$, define
\[
W(\bomega_r,\bomega_c)\;=\;\mathrm{Diag}(\bomega_r)\,\bW_0\,\mathrm{Diag}(\bomega_c),
\]
i.e., $\bW_{ij}=\bomega_{r,i}\,\bW_{0,ij}\,\bomega_{c,j}$. Consider the feasible set
\[
\Omega \;=\;\Big\{(\bomega_r,\bomega_c)\in\{0,1\}^m\times\{0,1\}^n:\ \|\bomega_r\|_0\ge 1- s,\ \|\bomega_c\|_0\ge 1 - s\Big\},
\]
where $s$ is the sparsity level and let $\mathcal{W}_\Omega:=\{W(\bomega_r,\bomega_c):(\bomega_r,\bomega_c)\in\Omega\}$.

We also consider its natural continuous relaxation $\widehat{\Omega}$:
\[
\widehat{\Omega}\;=\;\Big\{(\bomega_r,\bomega_c)\in[{\alpha},1]^m\times[{\alpha},1]^n:\ \|\bomega_r\|_0\ge 1 - s,\ \|\bomega_c\|_0\ge 1 - s\Big\}\!,
\]
with any fixed ${\alpha}\in(0,1]$ (so no active row/column is exactly zero during the first iteration of training).

\paragraph{Assumptions.}
\begin{enumerate}
\item \label{assump:W0-lowerbd} There exist active index sets $S_r\subseteq[m]$, $S_c\subseteq[n]$ with $|S_r|\ge 1 -s$, $|S_c|\ge 1 -s$ such that on $S_r\times S_c$ the base weights are bounded away from zero:
\[
\min_{i\in S_r,\; j\in S_c}\; |\bW_{0,ij}|\;\ge\; {w}_0 \;>\; 0.
\]
(Equivalently, $\bW_0$ has no vanishing entries on any row/column that the constraint allows to be active.)
\item \label{assump:domain} We optimize over the relaxed feasible set $\widehat{\Omega}$ (piecewise-smooth; the discrete case is covered as a limit).
\end{enumerate}

\begin{lemma}[Chain-rule lower bound via Jacobian]\label{lem:chain}
Let $F(\btheta):=L\!\big(W(\btheta)\big)$ where $\btheta:=(\bomega_r,\bomega_c)\in\widehat{\Omega}$. If, on $\widehat{\Omega}$,
\[
\sigma_{\min}\big(J(\btheta)\big)\;\ge\;\gamma \;>\; 0
\qquad\text{where } J(\btheta):=\frac{\partial\,\mathrm{vec}(W(\btheta))}{\partial \btheta}\in\mathbb{R}^{mn\times(m+n)},
\]
then $F$ satisfies the PL inequality on $\widehat{\Omega}$ with parameter $\beta\gamma^2$:
\[
\frac{1}{2}\,\|\nabla_{\btheta} F(\btheta)\|_2^2 \;\ge\; \beta\,\gamma^2 \,\big(F(\btheta)-L^\star\big), \qquad \forall\,\btheta\in\widehat{\Omega}.
\]
\emph{Proof.} By the chain rule, $\nabla_{\btheta} F(\btheta)= J(\btheta)^\top \nabla_{\mathrm{vec}(W)} L(W(\btheta))$. Hence
\[
\|\nabla_{\btheta} F(\btheta)\|_2 \;=\; \|J(\btheta)^\top \nabla_{\mathrm{vec}(W)} L\|_2
\;\ge\; \sigma_{\min}(J(\btheta))\,\|\nabla_{\mathrm{vec}(W)} L\|_2.
\]
Squaring, using $\sigma_{\min}(J)\ge\gamma$, the Frobenius/Euclidean equivalence, and \eqref{eq:PL-in-W} gives the claim. \qed
\end{lemma}

\begin{lemma}[Uniform lower bound on the Jacobian for gates]\label{lem:J-lower}
Under \ref{assump:W0-lowerbd}–\ref{assump:domain}, there exists $\gamma>0$ depending only on $(s,\underline{\alpha},\underline{w}_0)$ such that
$\sigma_{\min}(J(\btheta))\ge \gamma$ for all $\btheta\in\widehat{\Omega}$.
\end{lemma}
\begin{proof}
Write $\bW_{ij}=\bomega_{r,i}\,\bW_{0,ij}\,\bomega_{c,j}$. The partials are
\[
\frac{\partial \bW_{ij}}{\partial \bomega_{r,i}}=\bW_{0,ij}\,\bomega_{c,j},\qquad
\frac{\partial \bW_{ij}}{\partial \bomega_{c,j}}=\bomega_{r,i}\,\bW_{0,ij},\qquad
\text{others }0.
\]
Thus $J(\btheta)$ has two block-columns corresponding to $(\bomega_r,\bomega_c)$ whose nonzeros on $S_r\times S_c$ are bounded below in magnitude by $\underline{w}_0\,\underline{\alpha}$.
Moreover, the $mn\times(m+n)$ Gram matrix $G(\btheta):=J(\btheta)^\top J(\btheta)$ is block-structured:
\[
G(\btheta)=
\begin{bmatrix}
\mathrm{Diag}\!\big(g_r(\bomega_c)\big) & H(\btheta)\\
H(\btheta)^\top & \mathrm{Diag}\!\big(g_c(\bomega_r)\big)
\end{bmatrix},
\]
where $g_r(\bomega_c)_i=\sum_{j} (W_{0,ij}\bomega_{c,j})^2$ and $g_c(\bomega_r)_j=\sum_{i}(W_{0,ij}\bomega_{r,i})^2$.
On the active sets, $g_r(\bomega_c)_i\ge s\,\underline{w}_0^2\,\underline{\alpha}^2$ and
$g_c(\bomega_r)_j\ge s\,\underline{w}_0^2\,\underline{\alpha}^2$.
By Gershgorin theorem \citep{salas1999gershgorin}, the minimal eigenvalue of $G(\btheta)$ is bounded below by a constant $c(s,\underline{\alpha},\underline{w}_0)>0$ independent of $\btheta\in\widehat{\Omega}$ (dependent only on the diaogonal values). Hence $\sigma_{\min}(J(\btheta))=\sqrt{\lambda_{\min}(G(\btheta))}\ge \gamma:=\sqrt{c}>0$. \end{proof}

\begin{proposition}[PL is preserved under two-sided gates]\label{prop:PL-gates}
Since L satisfies PL condition, the composed objective
\[
F(\bomega_r,\bomega_c)\;:=\; L\!\big(\mathrm{Diag}(\bomega_r)\,\bW_0\,\mathrm{Diag}(\bomega_c)\big)
\]
satisfies the PL inequality on $\widehat{\Omega}$ with parameter $\beta\,\gamma^2$, where $\gamma>0$ is the uniform Jacobian lower bound from Lemma~\ref{lem:J-lower}.
\emph{Proof.} Combine Lemmas~\ref{lem:chain} and \ref{lem:J-lower}. \qed
\end{proposition}

\paragraph{Why the LoRA parametrization does \emph{not} preserve PL.}
Consider the same base loss $L$ but the LoRA map $W(\bA,\bB)=\bW_0+\bA\bB$ with $\bA\in\mathbb{R}^{m\times r},\,\bB\in\mathbb{R}^{r\times n}$.
Even if $L$ is PL in $W$ on $\mathbb{R}^{m\times n}$, the composed
\[
F_{\text{LoRA}}(\bA, \bB)\;:=\;L(\bW_0+\bA \bB)
\]
\emph{fails} the PL inequality in general.

\begin{proposition}[Counterexample for LoRA]\label{prop:lora-not-PL}
Let $L(\bW)=\tfrac12\|\bW-\bW^\star\|_F^2$, which satisfies the PL inequality on $\mathbb{R}^{m\times n}$ with parameter $\beta=1$. If $\bW^\star\neq \bW_0$, then $F_{\mathrm{LoRA}}(\bA, \bB)=\tfrac12\|\bW_0+\bA \bB-\bW^\star\|_F^2$ does not satisfy the PL inequality on $\mathbb{R}^{m\times r}\times\mathbb{R}^{r\times n}$.
\emph{Proof.}
At $(\bA,\bB)=(0,0)$ we have $\nabla_A F_{\mathrm{LoRA}}=(\bW_0+\bA \bB -\bW^\star)\bB^\top=0$ and $\nabla_B F_{\mathrm{LoRA}}=\bA^\top(\bW_0+\bA \bB-\bW^\star)=0$, so $\|\nabla F_{\mathrm{LoRA}}(0,0)\|_F=0$.
But $F_{\mathrm{LoRA}}(0,0)=\tfrac12\|\bW_0-\bW^\star\|_F^2>0$ by $\bW^\star\neq \bW_0$.
The PL inequality would force $\|\nabla F\|^2\ge 2\beta(F-F^\star)$, which is violated here since the left-hand side is $0$ while the right-hand side is positive (note $F^\star\ge 0$).
Hence $F_{\mathrm{LoRA}}$ is not PL. \qed
\end{proposition}

\section{Sparsification Results}
Sparsification-Accuracy trade-off measured on CoLA, SST2 (full), STSB, and MRPC datasets with RoBERTa-Base, RoBERTa-Large, and Llama3.2-1B backbones is shown in Figure \ref{fig:sparsity}.
\begin{figure*}[h]
\centering
\includegraphics[width=.24\linewidth]{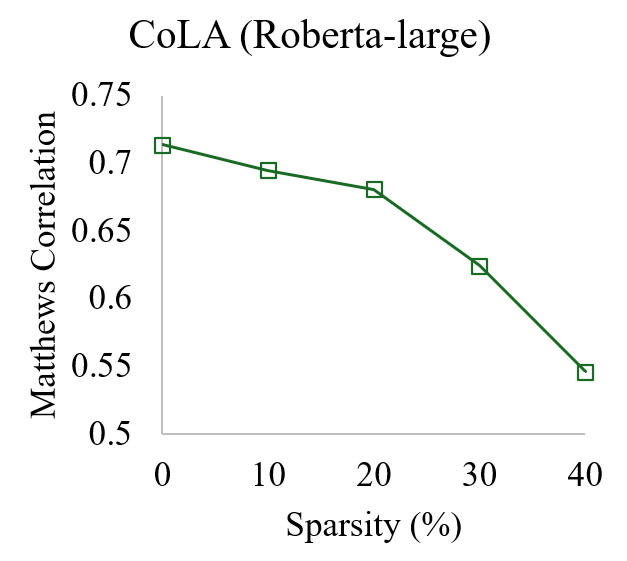}
\includegraphics[width=.24\linewidth]{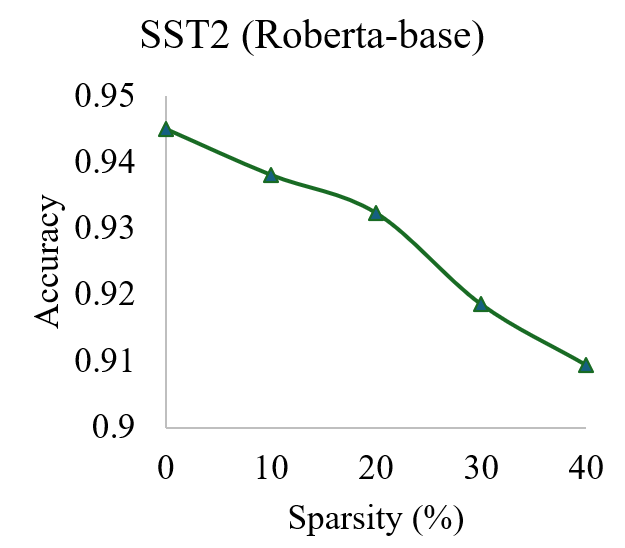}
\includegraphics[width=.24\linewidth]{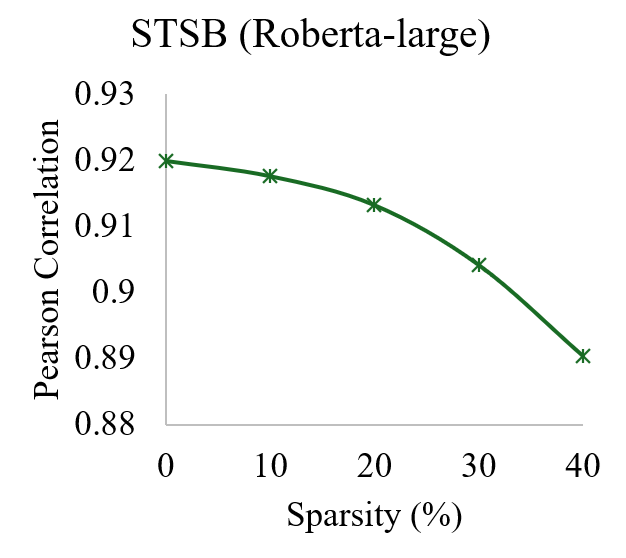}
\includegraphics[width=.24\linewidth]{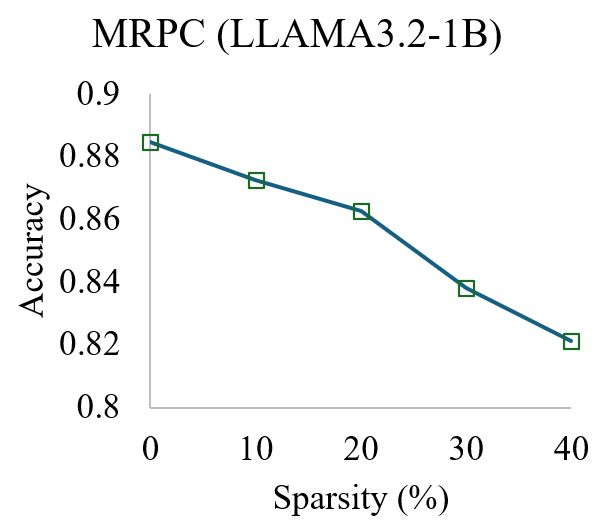}
\vskip -0.1in
\caption{Sparsification-Accuracy trade-off measured on CoLA, SST2 (full), STSB, and MRPC datasets with RoBERTa-Base, RoBERTa-Large, and Llama3.2-1B backbones. Our model provides $>\mathbf{40\%}$ of structured sparsity while sacrificing only $\mathbf{4\%}$ of accuracy compared to the model without sparsification on the SST2 dataset, where we train $\bm{\omega}_r, \bm{\omega}_c$ with total $166K$ parameters. On CoLA, the method reduces up to $20\%$ of parameters without significant loss in accuracy, and $40\%$ on the STSB dataset with only $3\%$ drop in accuracy. The method removes up to $~470M$ parameters from the Llama3.2-1B base model with only $6\%$ accuracy loss.
} 
\label{fig:sparsity}
% \vskip -0.2in
\end{figure*}

\section{Convergence Experiment}
We train both FineGates and LoRA on MRPC dataset for 50 epochs with a batch size of 16, learning rate $1-e10^{-3}$. In Figure \ref{fig:convergence} we show validation accuracy. We repeat the experiment 10 times and compute mean$\pm$std values after every epoch.

\section{FineGates Modifications}

We conduct additional experiments to examine whether adapting fewer projection matrices in the transformer layers significantly affects the performance of our approach. Furthermore, we investigate the effect of extending our model with additional low-rank weights \(\mathbf{W}_A\) and \(\mathbf{W}_B\) for each adapted layer, following the design proposed by \citep{hu2021lora}. Table \ref{tab:glue_combined_ablation} reports results for three model variants on RoBERTa-Base and RoBERTa-Large: FineGates w/o $\mathbf{W}_{mlp}$ that does not adapt intermediate and output projections for RoBERTa layers,  $\text{FineGates}$ w/ $\textbf{W}_B\textbf{W}_A$ adds low-rank matrices $\mathbf{W}_{B},\mathbf{W}_{A}$ with $r=8$ and $\text{FineGates}$ w $\textbf{W}_B\textbf{W}_A$ but w/o $\mathbf{W}_{mlp}$ which combines the two modifications. From this experiment, we draw three main conclusions: (1) the base FineGates model, where only gates are trained, achieves accuracy comparable to the other variants; (2) the number of trainable parameters can be reduced by up to 4 times without significantly compromising accuracy (FineGates w/o $\mathbf{W}_{mlp}$); and (3) introducing low-rank parameter matrices yields additional accuracy gains.

\begin{table*}[h]
\begin{center}
\caption{Modifications of FineGates. The evaluation is done for the next versions of FineGates: (1) training without gates on feed-forward layers (w/o $\mathbf{W}_{mlp}$), (2) training with additional low-rank parameters (w/ $\mathbf{W}_B \mathbf{W}_A$), (3) training with low-rank parameters (w/ $\mathbf{W}_B\mathbf{W}_A$) but excluding feed-forward layers (w/o $\mathbf{W}_{mlp}$) matrices from adaptation.}
\resizebox{.8\linewidth}{!}{
\begin{tabular}{ll|llll|llll} 
\hline
\multicolumn{2}{c}{} & \multicolumn{4}{c}{Small} & \multicolumn{4}{c}{Sub-sampled Large}\\
Method & TP $\downarrow$ & CoLA & STS-B & MRPC & RTE & SST2 & MNLI & QNLI & QQP \\ 
\hline 
\multicolumn{10}{c}{\textbf{Roberta-Base}} \\
\hline
FineGates & 0.17M & 65.7 & 91.0  & 90.0 & 83.4 & 94.0 & 81.3 & 89.1 & 84.9 \\
FineGates w/o $\mathbf{W}_{mlp}$ & 0.04M & 65.1 & 90.9  & \textbf{90.4} & 80.5 & 94.2 & 81.0 & 89.2 & 84.7 \\
\hline
FineGates w/ $\textbf{W}_B\textbf{W}_A$ & 1.4M & \textbf{66.5} &\textbf{91.2} & 90.2 & \textbf{83.8} & \textbf{94.3} & 81.4 & 89.4 & 84.8 \\
FineGates w/ $\textbf{W}_B\textbf{W}_A$ w/o $\mathbf{W}_{mlp}$ & 0.6M & 65.8 & 90.7 & 89.7 & 82.3 & \textbf{94.3} & \textbf{81.7} & 89.7 & \textbf{85.2} \\
\hline
\multicolumn{10}{c}{\textbf{Roberta-Large}} \\
\hline
FineGates & 0.4M & \textbf{71.4} & 92.3  & 91.2 & 90.2 & \textbf{96.2} & 86.2 & 92.4 & 86.1 \\
FineGates w/o $\mathbf{W}_{mlp}$ & 0.1M & 70.5 & 92.0  & 91.4 &\textbf{90.6} & 96.1 & 86.6 & 92.3 & 86.4 \\
\hline
FineGates w/ $\textbf{W}_B\textbf{W}_A$ & 3.8M & 70.1 & 92.2 & 90.6 & 88.1 & 95.4 & 86.1 & 92.1 & 85.5 \\
FineGates w/ $\textbf{W}_B\textbf{W}_A$ w/o $\mathbf{W}_{mlp}$ & 1.7M & 69.9 & \textbf{92.6} & \textbf{91.9} & 88.5 & \textbf{96.2} & \textbf{87.2} & \textbf{92.7} & \textbf{86.7} \\
\hline
\hline
\end{tabular}}

\label{tab:glue_combined_ablation}
\end{center}
\end{table*}

\section{Inference Time Measurements}
\label{app:inference_time}
We evaluate the inference time of the compressed model after each training epoch. For each epoch, the wall-clock time of the validation phase is measured 10 times per run. To account for variability, we train with three different random seeds, resulting in a total of 30 inference time measurements per epoch. All validation-time measurements are performed on CPU only. Our experiments use the MRPC dataset with the LLaMA-3.2-1B backbone. To assess inference efficiency, we compute the average validation time over a 5-epoch window and report the relative reduction in inference time using the first epoch as a reference. Specifically, the reduction is defined as $(t_0 - t_i)/t_0$, where $t_i$ denotes the average validation time after epoch $i$, aggregated over 30 measurements (3 seeds × 10 repetitions per seed). We report timing results alongside validation accuracy (averaged over three seeds) in Figure~\ref{fig:time_accuracy_mrpc}, and in relation to the number of pruned parameters in Figure~\ref{fig:time_params_mrpc}.

\begin{figure*}[h]
\centering
\includegraphics[width=.85\linewidth]{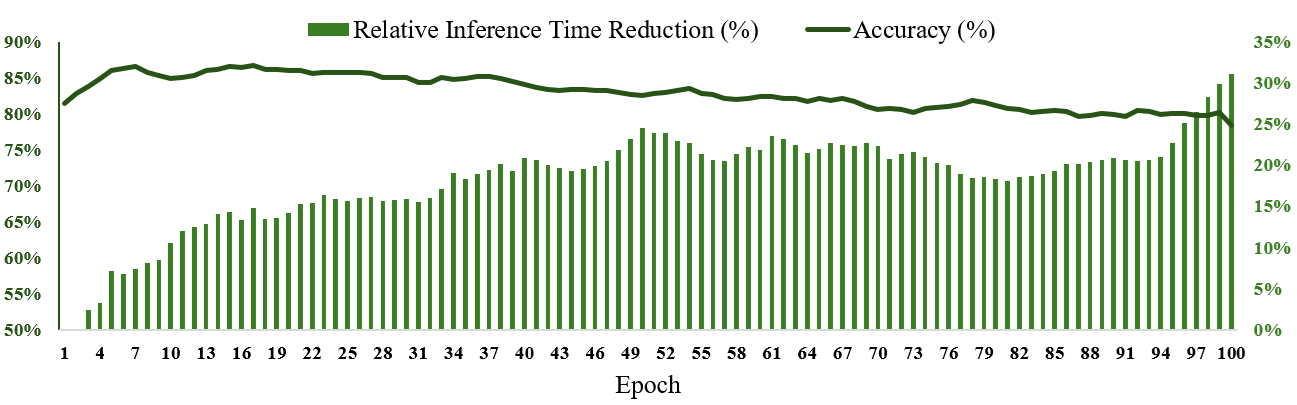}
\caption{CPU Inference time measurements along with validation accuracy averaged across 3 seeds for a single validation epoch of MRPC dataset.}
\label{fig:time_accuracy_mrpc}
\end{figure*}

\begin{figure*}[h]
\centering
\includegraphics[width=.85\linewidth]{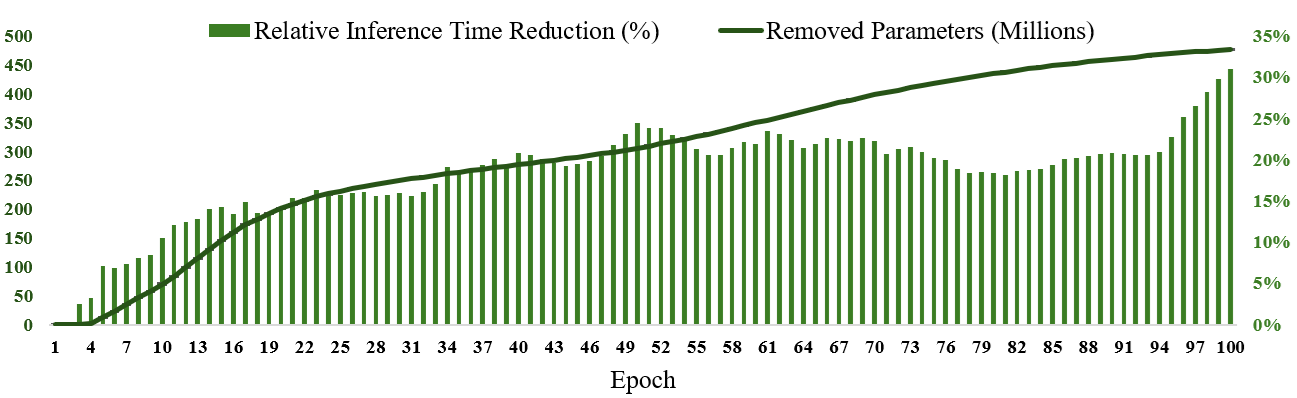}
\caption{CPU Inference time measurements along with number of removed parameters averaged across 3 seeds for a single validation epoch of the MRPC dataset.}
\label{fig:time_params_mrpc}
\end{figure*}

\section{Wall-clock Times measurements for APT and FineGates}
In Table \ref{tab:training_times} we show the total training time measured for APT and FineGates.

\begin{table}[h]
\centering
\caption{To compare wall-clock times, we trained both models on the SST-2 data for 50 epochs, measuring the total training time using the RoBERTa-base backbone.}
\resizebox{.3\linewidth}{!}{
\begin{tabular}{lc}
\toprule
\textbf{Model} & \textbf{Time (seconds)} \\ %& \textbf{Time (hours)} \\
\midrule
APT       & 37,515 \\ % & $\sim$10.4 \\
\rowcolor[HTML]{c9f2c4} FineGates & 35,010 \\ %& $\sim$9.7 \\
\bottomrule
\end{tabular}}
\vskip -0.1in

\label{tab:training_times}
\end{table}

\section{Proof of Lemma \ref{prop:two-sided-gates}}
\label{app::proof}
\begin{proof}
Write $\mathcal{L}(\bW,\bomega_r,\bomega_c):= f\!\big(\operatorname{Diag}(\bomega_r)\bW\operatorname{Diag}(\bomega_c)\big)$ and let
$G:=\nabla f(\widetilde{\bW})\in\mathbb{R}^{k\times d}$ evaluated at $\widetilde{\bW}=\operatorname{Diag}(\bomega_r)\bW\operatorname{Diag}(\bomega_c)$.
By the chain rule,
\[
\nabla_{\bW} F \;=\; \operatorname{Diag}(\bomega_r)\, G\, \operatorname{Diag}(\bomega_c)
\;=\; G \odot (\bomega_r\,\bomega_c^\top).
\]
For the gate variables, using
$\partial\widetilde{\bW}/\partial(\bomega_r)_i = e_ie_i^\top \bW\operatorname{Diag}(\bomega_c)$ and
$\partial\widetilde{\bW}/\partial(\bomega_c)_j = \operatorname{Diag}(\bomega_r)\bW e_je_j^\top$, we obtain
\begin{align*}
\big(\nabla_{\bomega_r} F\big)_i
&= \big\langle G,\, e_ie_i^\top \bW\operatorname{Diag}(\bomega_c)\big\rangle
 = \sum_{j=1}^d G_{ij}\, \bW_{ij}\, (\bomega_c)_j
 = \Big(\big(G\odot \bW\big)\,\bomega_c\Big)_i,\\
\big(\nabla_{\bomega_c} F\big)_j
&= \big\langle G,\, \operatorname{Diag}(\bomega_r)\bW e_je_j^\top\big\rangle
 = \sum_{i=1}^k G_{ij}\, \bW_{ij}\, (\bomega_r)_i
 = \Big(\big(G\odot \bW\big)^\top \bomega_r\Big)_j.
\end{align*}
Because $\nabla f$ is $L_f$-Lipschitz, the map $(\bW,\bomega_r,\bomega_c)\mapsto G$ is Lipschitz on any bounded set. Multiplication by $\operatorname{Diag}(\bomega_r)$ and $\operatorname{Diag}(\bomega_c)$ (with $\bomega_r,\bomega_c\in[0,1]$) is linear and non-expansive, hence the composite maps to $\nabla_{\bW} F$, $\nabla_{\bomega_r} F$, and $\nabla_{\bomega_c} F$ are Lipschitz on bounded sets. Since $h_r,h_c$ are smooth on $[0,1]$, their gradients are Lipschitz there; therefore,
\[
\nabla L \;=\; \nabla F \;+\; \lambda\big(0,\nabla h_r,\nabla h_c\big)
\]
is $L_L$-Lipschitz on a bounded set containing the iterates.

Let $\bz_t:=(\bW_t,\bomega_{r,t},\bomega_{c,t})$ and take one gradient step $\bz_{t+1}= \bz_t - \eta\nabla L(\bz_t)$ with $\eta\in(0,2/L_L)$. By the descent lemma,
\[
L(\bz_{t+1}) \;\le\; L(\bz_t) - \eta\Big(1-\tfrac{\eta L_L}{2}\Big)\,\|\nabla L(\bz_t)\|^2,
\]
so $L(\bz_t)$ is monotonically non-increasing. Since $f$ is bounded below and $h_r,h_c$ are bounded below on $[0,1]$, $L$ is bounded below; summing the inequality yields $\sum_t \|\nabla L(\bz_t)\|^2<\infty$, hence $\|\nabla L(\bz_t)\|\to 0$. Any accumulation point therefore satisfies the first-order optimality condition.
\end{proof}

\begin{remark}[Bias gating and one-sided special cases]
If a bias $b$ is gated as $\widetilde{b}=\bomega_r\odot b$ (or left ungated), the proof extends verbatim via standard column-augmentation.
Setting $\bomega_r=\mathbf{1}$ or $\bomega_c=\mathbf{1}$ recovers the one-sided gating case.
\end{remark}
\section{Kurtosis Score}
\label{appx:kurt}
Denote by $\mathbf{W} \in \mathbb{R}^{k \times d}$ a weight matrix selected from $\{ \mathbf{W}_q,\mathbf{W}_k,\mathbf{W}_v, \mathbf{W}_o, \mathbf{W}_{mlp}^i, \mathbf{W}_{mlp}^o\}$ in attention layer. For a given hidden representation $\bm{X} \in \mathbb{R}^{b \times t \times d}$, we compute the activation matrix $\bm{O} \in \mathbb{R}^{k \times d}$ by first averaging the tensor $\bm{X}$ in batch size and sequence length dimensions to obtain $ \mathbf{x}' \in \mathbb{R}^{1 \times d}$ by:
\begin{equation}
    \mathbf{x}'=\frac{1}{b \times t} \sum_{i=1}^{b}\sum_{j=1}^{t} \bm{X}_{i,j}
\end{equation}
and then compute the activation matrix by: 
\begin{equation}
\bm{O} =  \left[ \bm{\omega}_{r} \odot \bm{W} \odot \bm{\omega}_{c} \right] \odot \mathbf{x}'
\end{equation}

where $\odot$ denotes element-wise multiplication and $\bm{\omega}_{r},\bm{\omega}_{c}$  are gating vectors.

We then compute the kurtosis vectors $\bm{k}_{\text{cols}} = \text{kurt}(\bm{O}) \in \mathbb{R}^{1 \times d}$  and $\bm{k}_{\text{rows}} = \text{kurt}(\bm{O}^T) \in \mathbb{R}^{1 \times k}$ , where $\text{kurt}(\cdot)$ denotes Pearson’s kurtosis \citep{decarlo1997meaning}.

Note that aggregating tensor $\bm{X}$ by averaging over the batch and sequence dimensions is a simplification. A more accurate approach would compute kurtosis vectors for each batch and sequence element individually, following by averaging the vectors $\bm{k}_{\text{cols}}, \bm{k}_{\text{rows}}$. However, in practice, this computation is considerably more expensive and requires substantially more GPU memory.

\section{Statistics of the learned gates across different Transformer layers}
We analyze the statistics of the learned gates across different Transformer layers (Figure~\ref{fig:gates_stats}), and observe that the final layers tend to become sparse earlier in training. This behavior is consistent with prior findings, such as \citet{sanh2020movement} and \citet{gordon2020compressing}, which report that later layers are more amenable to pruning without significantly harming performance.

\begin{figure*}[h]
\centering
\includegraphics[width=.6\linewidth]{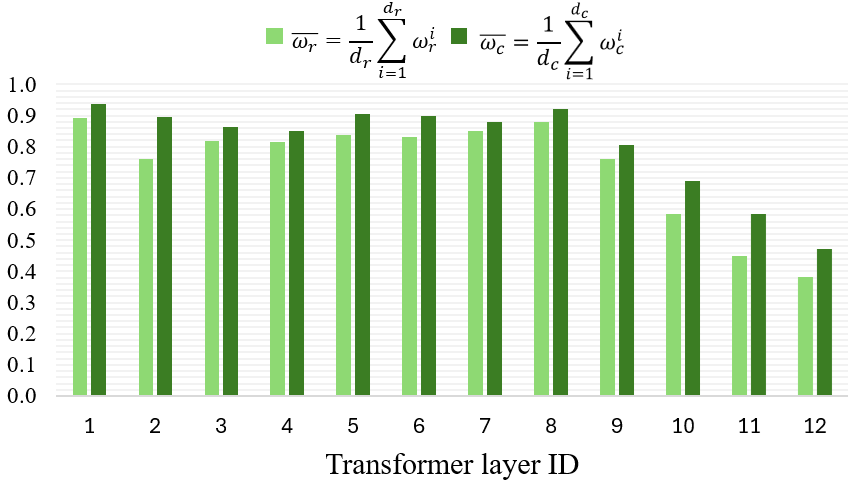}
\caption{For each Transformer layer in Roberta-Base model we present the gates values averaged for each gates row vector, $\frac{1}{d_r}\sum_{i=1}^{d_r} \bm{\omega}_r^i$, and  gates column vector $\frac{1}{d_c}\sum_{i=1}^{d_c} \bm{\omega}_c^i$, where $d_r, d_c$ are the dimensions of gates vectors $\bm{\omega}_r, \bm{\omega}_c$ correspondingly.}
\label{fig:gates_stats}
\end{figure*}

\section{LLAMA-1B Model}

\begin{table}[h!]
\centering
\caption{Model architecture for the pretraining experiment.}
\resizebox{.7\linewidth}{!}{
\begin{tabular}{rrrrrrr}
\toprule
\textbf{Params} & \textbf{Hidden} & \textbf{Intermediate} & \textbf{Heads} & \textbf{Layers} & \textbf{Steps} & \textbf{Data Amount} \\
\midrule
1B    & 2,048 & 5,461  & 24  & 32  & 20K & 2B \\
\bottomrule
\end{tabular}}

\label{tab:pretrain_model}
\end{table}

We set the maximum sequence length to 256 for all models, with a batch size of 512. We apply learning-rate warmup for the first 10,000 training steps, followed by a cosine annealing schedule that decays to 10\% of the initial learning rate. Training is stopped after processing 2B tokens from the training dataset. The model architecture is summarized in Table \ref{tab:pretrain_model}, and we use the T5-base tokenizer for this experiment.

\section{MaskLLM Evaluation}
We evaluate the pre-trained model using the original implementation provided by the authors\footnote{\text{https://github.com/NVlabs/MaskLLM}}
. The masking model we tested was trained by the authors on a subset of the C4 dataset. While the original results were reported with a maximum sequence length of 4096, due to our limited computational budget we evaluate it with a maximum sequence length of 256. For a fair comparison, we train FineGates on the C4 dataset using the same sequence length.

\section{GLUE Benchmark statistics}
Presented in Table ~\ref{tab:glue}.

\begin{table*}[h]
\begin{center}
\caption{The GLUE benchmark datasets statistics}
\resizebox{.75\linewidth}{!}{
\begin{tabular}{cccccccccc} 
\hline
\textbf{Dataset} & MNLI & QQP & QNLI & SST2 & COLA & STSB & MRPC & RTE \\
\hline
\textbf{Samples} & 392,702 & 363,846 & 104,743 & 67,349 & 8,551 & 5,749 & 3,668 & 2,490 \\
\hline
\end{tabular}}

\label{tab:glue}
\end{center}
\end{table*}

\section{Hyperparameters for training}
We present the hyperparameters for RoBERTa models in Table~\ref{tab:glue_hyperparams}, for comparisons with the APT model in Table~\ref{tab:roberta_base_hyperparams}, and for LLaMA experiments in Table~\ref{tab:llama_hyperparams}. Each dataset includes its own validation set, which is used for evaluation. In target-task fine-tuning experiments, we train all models with the Adam optimizer and decoupled weight decay regularization \citep{loshchilov2017decoupled}, optimizing $\bm{\Omega}$ with fixed learning rates across tasks. For sparsified base-model pretraining and post-training pruning, we restrict the training budget to 2B tokens and apply learning-rate warmup followed by a cosine scheduler.

\begin{table*}[t]
\centering
\caption{Hyperparameters for Table 1.}
\resizebox{.8\linewidth}{!}{
\begin{tabular}{lcccccccc}
\toprule
\textbf{Hyperparameter / Dataset} & \textbf{CoLA} & \textbf{STS-B} & \textbf{MRPC} & \textbf{RTE} & \textbf{SST-2} & \textbf{MNLI} & \textbf{QNLI} & \textbf{QQP} \\
\midrule
\multicolumn{9}{c}{RoBERTa-Base} \\
 lr bias + head & 1e-4 & 1e-4 & 1e-4 & 1e-4 & 1e-4 & 1e-4 & 1e-4 & 1e-4 \\
 lr gates        & 1e-3 & 1e-3 & 1e-3 & 1e-3 & 1e-3 & 1e-3 & 1e-3 & 1e-3 \\
 lr schedule     & constant & constant & constant & constant & constant & constant & constant & constant \\
 epochs          & 100 & 100 & 100 & 100 & 50 & 50 & 50 & 50 \\
 lambda          & 50 & 10 & 50 & 50 & 0.1 & 0.01 & 0.01 & 0.01 \\
 max seq length  & 512 & 512 & 512 & 512 & 512 & 512 & 512 & 512 \\
\midrule
\multicolumn{9}{c}{RoBERTa-Large} \\
 lr bias + head & 1e-4 & 1e-4 & 1e-4 & 1e-4 & 1e-4 & 1e-4 & 1e-4 & 1e-4 \\
 lr gates        & 1e-3 & 1e-3 & 1e-3 & 1e-3 & 1e-3 & 1e-3 & 1e-3 & 1e-3 \\
 lr schedule     & constant & constant & constant & constant & constant & constant & constant & constant \\
epochs          & 50 & 50 & 50 & 50 & 100 & 100 & 100 & 100 \\
lambda          & 10 & 10 & 10 & 10 & 0.1 & 0.1 & 0.1 & 0.1 \\
max seq length  & 512 & 512 & 512 & 512 & 512 & 512 & 512 & 512 \\
\bottomrule
\end{tabular}}

\label{tab:glue_hyperparams}
\end{table*}
\begin{table*}[h]
\centering
\caption{Hyperparameters for Table 2.} %\ref{tab:sparsity_apt}.}
\resizebox{.65\linewidth}{!}{
\begin{tabular}{llcccc}
\toprule
 \textbf{Hyperparameter / Dataset} & \textbf{CoLA} & \textbf{STS-B} & \textbf{MRPC} & \textbf{RTE} & \textbf{SST-2} \\
\midrule
\multicolumn{6}{c}{RoBERTa-Base} \\
lr bias + head & 1e-4 & 1e-4 & 1e-4 & 1e-4 & 1e-4 \\
lr gates        & 1e-3 & 1e-3 & 1e-3 & 1e-3 & 1e-3 \\
lr schedule     & constant & constant & constant & constant & constant \\
epochs          & 100 & 100 & 100 & 100 & 50 \\
lambda          & 100 & 100 & 100 & 100 & 100 \\
max seq length  & 512 & 512 & 512 & 512 & 512 \\
\bottomrule
\end{tabular}}

\label{tab:roberta_base_hyperparams}
\end{table*}
\begin{table*}[h]
\centering
\caption{Hyperparameters for Table 3.} %\ref{tab:llama1b_glue}.}
\resizebox{.65\linewidth}{!}{
\begin{tabular}{llcccc}
\toprule
\textbf{Hyperparameter/ Dataset} & \textbf{CoLA} & \textbf{STS-B} & \textbf{MRPC} & \textbf{RTE} & \textbf{SST-2} \\
\midrule
\multicolumn{6}{c}{LLAMA-3.2-1B} \\
lr bias + head & 1e-4 & 1e-4 & 1e-4 & 1e-4 & 1e-4 \\
lr gates        & 1e-3 & 1e-3 & 1e-3 & 1e-3 & 1e-3 \\
lr schedule     & constant & constant & constant & constant & constant \\
epochs          & 100 & 100 & 100 & 100 & 50 \\
lambda          & 0.01 & 0.01 & 0.01 & 0.01 & 0.01 \\
max seq length  & 512 & 512 & 512 & 512 & 512 \\
\bottomrule
\end{tabular}}

\label{tab:llama_hyperparams}
\end{table*}

% \section{Sprase Model Preraining}
\begin{figure}[t]
\centering
\includegraphics[width=.7\linewidth]{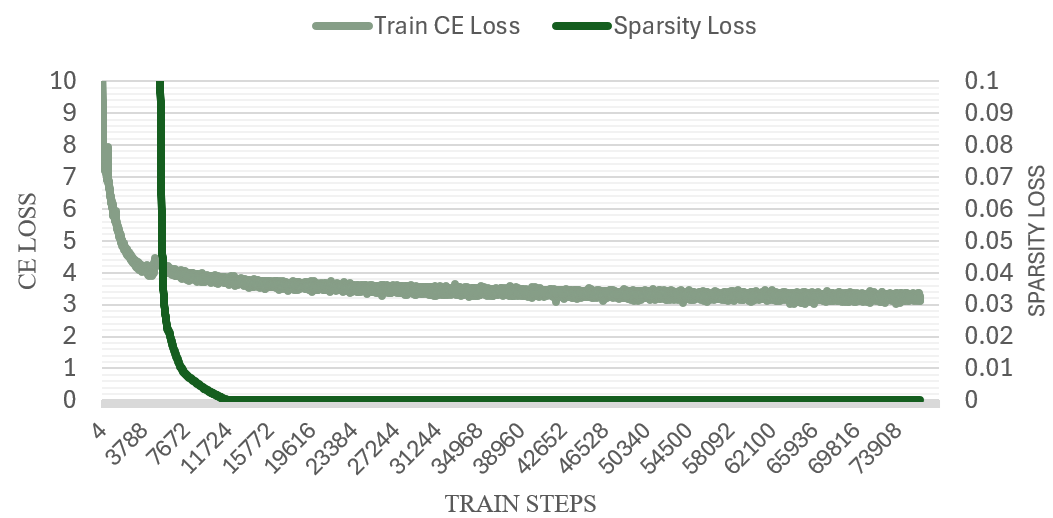}
\caption{Train loss convergence of FineGates + FFT method in Section 5.7}
\label{fig:pretrain_train}
\end{figure}

\begin{figure}[t]
\centering
\includegraphics[width=.7\linewidth]{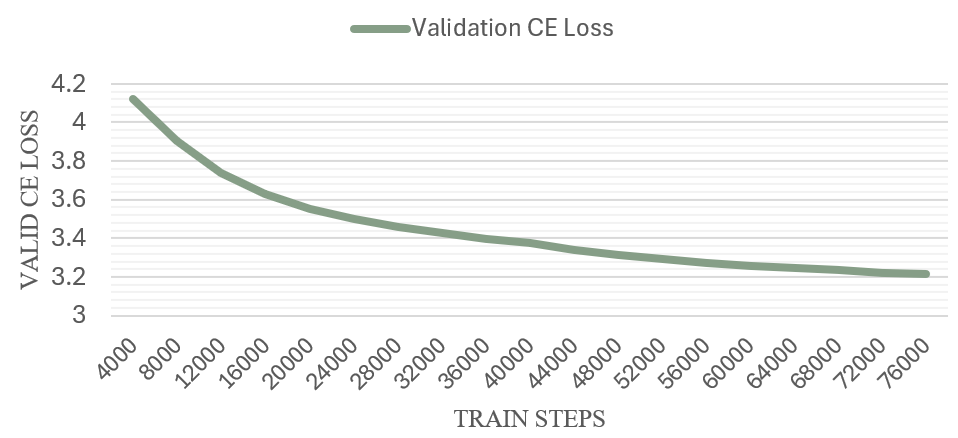}
\caption{Validation loss convergence of FineGates + FFT method in Section 5.7}
\label{fig:pretrain_validation}
\end{figure}

\section{Empirical Loss Convergence Plots}
To support our theoretical claim regarding model convergence, we present the loss curves of the pre-training LLaMA-1B base model with FineGates in Figures \ref{fig:pretrain_train} and \ref{fig:pretrain_validation}.

\section{Hardware used for experiments.}
We conducted training and evaluation of the compressed LLaMA-3.2-1B backbone model on a single NVIDIA RTX 6000 GPU paired with an AMD EPYC 9334 32-core processor. While training was performed on the GPU, evaluation wall-clock time measurements were obtained on the CPU. For all other experiments, we used a single NVIDIA H200 GPU with an Intel® Xeon® Platinum 8568Y+ CPU for both training and evaluation.

\section{Limitations}
\label{app:limitations}
While our method demonstrates strong performance and efficiency gains, it is important to acknowledge certain limitations. Our experiments are restricted to the GLUE benchmark. This reflects the practical constraints faced by smaller academic labs, which often lack access to large-scale computing resources. Although we could not evaluate larger models such as LLaMA-2 13B or GPT-style architectures, we contend that results from smaller-scale models remain valuable. They provide rigorous, reproducible insights and can inform innovations that may later scale to larger systems. Moreover, GLUE offers a standardized benchmark for comparison with prior work, making it a meaningful testbed for early-stage progress, and is still used by many recent works, i.e. \citet{zhang-etal-2025-parameter},\citet{lei2024fast},\citet{wang2024lora},\citet{refael2024adarankgrad},\citet{balazy2024lora},\citet{ceritli2024study}, \citet{ miles2024velora}.

\section{Baselines Description}
\label{app:baselines}
We compare our method against full fine-tuning—where the model is initialized with pre-trained weights and all parameters are updated during training—as well as several recently proposed efficient fine-tuning methods:
\begin{itemize}
    \item LoRA \citep{hu2021lora} with rank$=4$ applied additionally to the weight matrices $\widetilde{W}_{\text{mlp}}$ as in our method, which results in total $0.7$M trainable parameters for Roberta-Base model and $1.8$M for Roberta-Large base model.
    
    \item BitFit \citep{zaken2021bitfit} - an efficient finetuning method where only biases are trained for each adapted layer in the base model.
    \item  VeRA \citep{kopiczko2023vera} - vector-based random matrix adaptation where the number of trainable parameters is significantly lower than in LoRA without pruning or compression of the base model.
    \item LoRA-XS \citep{balazy2024lora} further reduces the number of trainable parameters by learning a small matrix positioned between two frozen low-rank matrices.
    \item RoCoFT$_{\text{1-Row}}$ \citep{kowsher2024rocoft} - the method proposed to directly train a single row or column in the base model without adding adapter weights. 
    \item APT \citep{zhao2024apt} - a recently proposed finetuning method that allows compression of the base model while training adapter weights.  Our approach differs from APT in two main points: (1) we learn the gates for the base model weights jointly with the finetuning task objective, and (2) our model obtains comparable results without training additional low-rank matrices $\mathbf{W}_A, \mathbf{W}_B$ resulting in optimization of less trainable parameters.
    \item VeLoRA \citep{miles2024velora} an adapter-based method similarly to LoRA which is trained with low gradients projections.

    \item MaskLLM \citep{fang2024maskllm} is a method that prunes the LLM models in a semi-structured (or ``N:M'') way, aimed at reducing computational overhead during inference. Instead of developing a new importance criterion, MaskLLM explicitly models N:M patterns as a learnable distribution through Gumbel Softmax sampling.
    
\end{itemize}

\end{document}